\newif\ifprivate                
\newif\ifconf                   
\newif\ifshort                  
\newif\ifcompress               
\newif\ifuc                     
\newif\iffinal                  

\privatefalse\conffalse\shortfalse\compressfalse\ucfalse\finaltrue

\ifconf

\relax
\documentclass[letterpaper]{article} 
\usepackage{aaai19}  
\usepackage{times}  
\usepackage{helvet}  
\usepackage{courier}  
\usepackage{url}  
\usepackage{graphicx}  
\usepackage{xcolor}
\usepackage{soul}
\usepackage[utf8]{inputenc}
\usepackage[small]{caption}

\def\eqvsp{}  \newdimen\paravsp  \paravsp=1.3ex 

\ifcompress\iffinal\else 
\usepackage{times}
\advance\textwidth by 1mm
\advance\textheight by 1mm
\def\eqvsp{\vspace{-0.5ex}}
\fi\fi

\else 

\documentclass[12pt,twoside]{article}
\def\eqvsp{}  \newdimen\paravsp  \paravsp=1.3ex 
\usepackage[numbers]{natbib}

\fi

\usepackage{enumerate}
\usepackage{latexsym,amsmath,amssymb,amsthm}
\usepackage{wrapfig}  
\usepackage{subcaption}
\usepackage{graphicx} 
\usepackage[all]{xy}  
\usepackage{tikz}     
\usepackage{pgf}
\usetikzlibrary{arrows,automata,positioning, fit, shapes.geometric, calc,patterns}

\def\,{\mskip 3mu} \def\>{\mskip 4mu plus 2mu minus 4mu} \def\;{\mskip 5mu plus 5mu} \def\!{\mskip-3mu}
\def\dispmuskip{\thinmuskip= 3mu plus 0mu minus 2mu \medmuskip=  4mu plus 2mu minus 2mu \thickmuskip=5mu plus 5mu minus 2mu}
\def\textmuskip{\thinmuskip= 0mu                    \medmuskip=  1mu plus 1mu minus 1mu \thickmuskip=2mu plus 3mu minus 1mu}
\ifprivate\else\iffinal\else\def\dispmuskip{}\def\textmuskip{}\fi\fi
\textmuskip
\def\beq{\eqvsp\dispmuskip\begin{equation}}    \def\eeq{\eqvsp\end{equation}\textmuskip}
\def\beqn{\eqvsp\dispmuskip\begin{displaymath}}\def\eeqn{\eqvsp\end{displaymath}\textmuskip}
\def\bqa{\eqvsp\dispmuskip\begin{eqnarray}}    \def\eqa{\eqvsp\end{eqnarray}\textmuskip}
\def\bqan{\eqvsp\dispmuskip\begin{eqnarray*}}  \def\eqan{\eqvsp\end{eqnarray*}\textmuskip}

\newtheorem{theorem}{Theorem}

\newtheorem{lemma}[theorem]{Lemma}

\newtheorem{assumption}[theorem]{Assumption}
\ifconf\else
\newenvironment{keywords}{\centerline{\bf\small
    Keywords}\begin{quote}\small}{\par\end{quote}\vskip 1ex}
\fi

\newtheorem{myexample}[theorem]{Example} 

\def\paradot#1{\vspace{\paravsp plus 0.5\paravsp minus 0.5\paravsp}\noindent{\bf\boldmath{#1.}}} 

\def\eps{\varepsilon}           
\ifprivate\def\todo#1{{\color{brown}\it[[ToDo:\ #1]]}}\else\def\todo#1{}\fi 
\ifprivate\def\note#1{{\color{olive}\it[[Note:\ #1]]}}\else\def\note#1{}\fi 
\def\SetR{\mathbb{R}}           
\def\B{\{0,1\}}                 
\def\P{{\rm P}}                 

\def\d{\delta}
\def\g{\gamma}

\def\A{{\cal A}}                
\def\S{{\cal S}}
\def\H{{\cal H}}
\def\B{{\cal B}}
\def\O{{\cal O}}
\def\R{{\cal R}}
\def\ta{{\tilde{a}}}
\def\tb{{\tilde{b}}}
\def\th{{\tilde{h}}}
\def\ts{{\tilde{s}}}
\def\tr{{\tilde{r}}}

\def\to{{\tilde{o}}}
\def\pa{{a^\prime}}

\def\ph{{h^\prime}}
\def\ps{{s^\prime}}
\def\pr{{r^\prime}}
\def\po{{o^\prime}}

\def\DPi{\eps_\Pi}
\def\DQ{\eps_Q}
\def\piR{{\pi_R}}

\def\abs#1{\left|#1\right|}

\newcommand{\norm}[1]{\left\lVert#1\right\rVert}

\newcommand{\AxisRotator}[1][rotate=0]{%
    \tikz [x=0.25cm,y=0.60cm,line width=.2ex,-stealth,#1] \draw (0,0) arc (-150:150:1 and 1);%
}

\ifshort
\def\citeay#1{\citeauthor{#1} \shortcite{#1}}
\else
\def\citeay#1{\citet{#1}}
\fi

\begin{document}

\def\mytitle{Performance Guarantees for Homomorphisms Beyond Markov Decision Processes}

\ifconf

\title{\mytitle}
\author{    
    Sultan Javed Majeed$^1$,
    Marcus Hutter$^2$
    \\ 
    $^{1,2}$ Research School of Computer Science, Australian National University, Australia\\
    $^{1}$ http://www.sultan.pk,
    $^{2}$ http://www.hutter1.net
}

\maketitle

\else

\title{\vspace{-4ex}
    \iffinal\else\normalsize\sc Technical Report \hfill Draft \hfill \ifprivate With Private Comments \else Confidential \fi \fi
    \vskip 2mm\bf\Large\hrule height5pt \vskip 4mm
    \mytitle \vskip 4mm \hrule height2pt}

\author{
    \bf{Sultan Javed Majeed$^1$,
    Marcus Hutter$^2$}
    \\ [3mm]
    $^{1,2}$ Research School of Computer Science, ANU, Australia
    \iffinal\thanks{A shorter version appeared in the proceedings of the AAAI 2019 conference.}\fi
    \\
    $^{1}$ http://www.sultan.pk,
    $^{2}$ http://www.hutter1.net
}

\iffinal\date{10 November 2018}\fi 

\maketitle

\vspace*{-5ex}

\fi

\begin{abstract}
    
    Most \emph{real-world} problems have huge state and/or action spaces. Therefore, a naive application of existing \emph{tabular} solution methods is not tractable on such problems. Nonetheless, these solution methods are quite useful if an agent has access to a relatively small state-action space homomorphism of the true environment and near-optimal performance is guaranteed by the map. 
    A plethora of research is focused on the case when the homomorphism is a Markovian representation of the underlying process. However, we show that near-optimal performance is sometimes guaranteed even if the homomorphism is non-Markovian. 
    \ifshort\else Moreover, we can aggregate significantly more states by lifting the Markovian requirement without compromising on performance. In this work, we expand Extreme State Aggregation (ESA) framework to joint state-action aggregations. We also lift the policy uniformity condition for aggregation in ESA that allows even coarser modeling of the true environment. \fi
    
    \ifconf\else\vspace{5ex}\def\contentsname{\centering\normalsize Contents}\setcounter{tocdepth}{1}
    {\parskip=-2.7ex\tableofcontents}\fi
\end{abstract}

\ifconf\else\vspace*{-2ex}
\begin{keywords} 
    homomorphism, state aggregation, non-MDP, action-value aggregation, reinforcement learning.
\end{keywords}
\fi

\ifconf\else\newpage
\fi

\ifprivate\it
\newpage\rm\fi

\section{Introduction}\label{sec:Intro}

The task of learning a near-optimal behavior from a sequence of experiences can naturally be formulated as a Reinforcement Learning (RL) problem \cite{Sutton1998}. In a typical RL framework, an agent interacts with an environment by taking an action and receiving a feedback. 
\ifshort\else The agent strives to maximize the expected positive feedback from the environment over time. \fi

It is typically assumed that the agent is facing a small state-action space\footnote{We refer a state and action space/pair jointly as a \emph{state-action} space/pair.} Markov Decision Process (MDP) so the agent can advise a stationary policy as a function of state \cite{Puterman2014}. Unfortunately, the number of state-action pairs in most of real-world problems is prohibitively large, e.g. driving a car, playing Go, personal assistance, controlling a plant with real-valued inputs, and so forth. The agent can neither simply visit each state-action pair nor can it keep record of these visits to learn a near-optimal behavior. This explosion of state-action space is known as the \emph{curse of dimensionality} \cite{Sutton1998}. Therefore, it is essential for the agent to generalize over its experiences in such a huge state-action space problem.

The curse of dimensionality is not a mere artifact of limited experience and computation constraints if the problem has an infinite state-action space. In a typical General Reinforcement Learning\footnote{We formally define a typical GRL setup in Section \ref{sec:grl-setup}.} (GRL) framework the agent is faced with an environment without any known structure. The GRL setup is arguably the most general setup: it can represent MDP, $k-$MDP, partially observed MDP (POMDP) and other typical environment models \cite{Hutter2016,Leike2016b}. But this generality comes at the cost of  an infinite state-action space: every agent-environment interaction generates a \emph{unique} history. Hence, there is no other option but to consider every history as a unique state of the environment. Therefore, GRL suffers, inevitably, by the curse of dimensionality.

\ifshort\else Solving the curse of dimensionality is an active field of research. A typical solution is to provide (or learn) a finite state-action space model such that the agent can perform near-optimal in the true environment. State aggregation, linear function approximation \cite{Parr2008}, neural networks --- a type of non-linear function approximation --- \cite{Sutton1998} are some well-celebrated solutions to the aforementioned problem. These methods are collectively called generalization, aggregation or dimensionality reduction methods. We are going to refer to them loosely as abstraction/aggregation procedures in this work.\fi

\ifshort\else Although most of the abstraction proposals concentrate on the state space reduction \cite{Abel2016,Li2006}, there is another equally important dimension of action space that hinders the application of traditional RL methods to real-world problems. The problem with a small state but huge --- sometimes continuous --- action space is equally challenging for learning and planning, {\em cf.} continuous bandit problem \cite{Bubeck2012}.  \fi

A \emph{homomorphism} framework originated by \citeay{Whitt1978} is a well-studied solution to handle the state-action space curse of dimensionality. In the homomorphism framework a problem of a large state-action space is \emph{solved} by using an \emph{abstract} problem with a relatively small state-action space.
The \mbox{(near-)optimal} policy of the abstract problem is a \emph{solution} if it is also a \mbox{(near-)optimal} policy in the true environment.

It is important to highlight that homomorphism is not the only technique for abstracting actions. The \emph{options} framework is a competing method for temporal action abstractions \cite{Sutton1999}. In the option/macro-action framework, the original action space is augmented with long term/built-in policies \cite{McGovern1997}. The agent using an option/marco-action commits to execute a fixed set of actions for a fixed (expected) time duration. This temporal action abstraction framework is arguably more powerful but beyond the scope of this work. Because, to the best of our knowledge, there are no theoretical performance guarantees available for such methods, and most probably such bounds might not exist.

In the homomorphism framework, it is typically assumed that the abstract problem is an MDP \cite{Ravindran2003,Ravindran2004,Taylor2008}. However, the size of the abstract state-action space can be significantly reduced if non-MDP abstractions are possible \cite{Abel2016,Li2006}. Moreover, the reduction of abstract state-action space roughly\footnote{Although reduction of state-action space is necessary for faster learning/planning but not sufficient \cite{Littman1995}.} translates into faster learning and planning \cite{Strehl2009,Lattimore2014}. 

\ifshort\else The model of the environment can be a crude approximation of the true environment: mapping every history to a single state, a fixed discretization of continuous action space and so on. Therefore, a coarse reduction is not a direct measure of quality of the abstraction: the agent can be arbitrarily worse off in the true environment. Nevertheless, we prefer short maps, i.e. small aggregated state-action space, since the agent can learn faster as compared to a non-aggregation based agent: it has more data available to learn with the same amount of experience. But, it is only possible if the abstraction has a bounded performance loss guarantee, i.e. the agent does not loose much if it considers the aggregated model as the true environment. \fi

It has recently been shown that the MDP restriction is not a necessary condition for near-optimal performance guarantees in state-only abstractions of GRL \cite{Hutter2016}. In this work, we use similar notation and techniques of \citeay{Hutter2016} but investigate and prove optimality bounds for non-MDP state-action homomorphisms in GRL. Since state abstraction is a special case of homomorphism (the action space is not reduced/mapped), our work is a generalization of Extreme State Aggregation (ESA) \cite{Hutter2016}.

The homomorphism framework has been extended beyond MDPs to finite-state POMDPs \cite{Wolfe2010}. As mentioned earlier,  GRL has an infinite set of histories and no two histories are alike. We can represent a finite-state POMDP environment as a history-based process by imposing a structure that there is an internal MDP that generates the observations and rewards. The GRL framework, by design, is more powerful and expressive than a finite-state POMDP \cite{Wolfe2010,Leike2016b}. Therefore, our results are more general than finite-state POMDP homomorphisms.

\ifshort\else \paradot{Contributions} We expand ESA performance guarantees to joint state-action aggregations that scale beyond MDPs. We make another important technical contribution by relaxing the policy uniformity condition in ESA. In ESA, the states are aggregated together if they have the same policy. We show that this requirement can be relaxed and states with approximately similar policies can also be aggregated together with little performance loss. It enables us to have near-optimal maps with, considerably, coarsely aggregated state-action pairs. \fi

\ifshort\else The rest of the paper is structured as follows. Section~\ref{sec:Preliminary} lays the foundations of our homomorphism setup. In Section~\ref{sec:Example}, we motivate the importance of non-MDP homomorphism by an example. Section~\ref{sec:key-elements} contains the key elements required to prove the main results. Section~\ref{sec:Aggregation} contains the main results of this work. In Section~\ref{sec:Disc}, we discuss and conclude the paper. 
We provide all the omitted proofs and lemmas in Appendix~\ref{sec:proofs} and \ref{sec:lemmas}, respectively. Appendix~\ref{sec:notation} provides a detailed list of notation. Section~\ref{sec:Example-2} provides another example non-MDP homomorphism.
\fi

\section{Preliminaries}\label{sec:Preliminary}

This section provides the required notation, a typical GRL framework and our homomorphism setup.
We consider a simple agent-environment setup \cite{Sutton1998}. 
The agent has a finite set of actions $\A$. The environment receives an action from the agent and gives a standard observation from a finite set of observations $\O$ and a real-valued reward from a finite set $\R \subseteq \mathbb{R}$. 
The agent interacts with the environment in cycles and in each cycle the agent performs an action and receives an observation and reward from the environment. This agent-environment interaction generates a possibly infinite history from an infinite set of histories $\H := \bigcup_{t=0}^{\infty} (\A \times \O \times \R)^t$. Hence, the original state-action space is the history-action space\footnote{In general, histories are considered as the states of the environment, so we interchangeably call the history-action space the \emph{original} state-action space.}, i.e. $\H \times \A$.
Similarly, we define an abstract finite state space $\S$ and action space $\B$ to form the abstract state-action space, i.e. $\S \times \B$.

We use a consistent notation throughout the paper unless otherwise stated. We use $\Delta(\cdot)$ to denote a probability distribution over its argument, 
$\norm{\cdot}_1$ expresses a first norm, $\tilde{x}$ is a local variable and $x^\prime$ is a different member of the same set. We use a shorthand notation $\forall f(x)=y$ to imply $\forall x, y: f(x)=y$. 
\ifshort\else A detailed list of notation is provided in Appendix~\ref{sec:notation}. \fi
We often make references to the results presented later in the paper. The reader is not encouraged to follow these \emph{justifying} references in the first reading.

\subsection{General Reinforcement Learning Framework} \label{sec:grl-setup}
This section provides a formal layout of a typical GRL framework and some assumptions we make about the setup. We start our setup by defining two center pieces of any RL setup: the environment and the agent/policy\footnote{While it can be argued that an agent and a policy are two separate entities, in this work we use them interchangeably.}. The environment, also referred as the \emph{original process} $P$, is defined as a stochastic mapping from a history-action pair to a distribution over the observation-reward pairs, i.e. $P : \H \times \A \rightarrow \Delta(\O \times \R)$. The \emph{history-based} agent/policy $\Pi$ is defined to be a function that stochastically maps a history to the actions as $\Pi : \H \rightarrow \Delta(\A)$.

\ifshort\else The agent has to maximize the expected sum of rewards it gets from the environment. But this sum can diverge if the agent keeps on adding infinitely many positive rewards. There are, at least, two options to rectify this problem by either assuming a finite sum of rewards or letting the agent discount its future rewards. \fi

\begin{assumption}{$(\mathrm{Geometric \ discounting})$} We assume a geometric discounting over the rewards --- i.e. the agent discounts its future rewards by a constant discount factor $\g \in [0, 1)$.
\end{assumption}

\ifshort\else This discounting of rewards serves a dual purpose for the agent, first, it eliminates the problem of infinite sum and second it serves as a parameter for the effective future rewards the agent should care about. A small discount factor makes the agent short sighted and a large discount factor lets the agent be more concerned about future rewards. \fi
The goal of the agent is to maximize the expected discounted sum of rewards which is generally expressed with Bellman equations of \mbox{(action-)value} functions \cite{Sutton1998}. The agent tries to maximize this value function and strives to reach the most valuable states. We define the action-value function $Q^\Pi$ for any history $h \in \H$ and action $a \in \A$ as
\beq
Q^\Pi(h,a) := \sum_{\to \in \O, \tr \in \R} P(\to \tr |ha)\left(\tr + \g V^\Pi(\th)\right)
\eeq
where $\th := ha\to \tr$ is an extended history and the corresponding value function $V^\Pi$ is defined as
\beq
V^\Pi(h) := \sum_{\ta \in \A} Q^\Pi(h,\ta) \Pi(\ta|h).
\eeq

The \mbox{(action-)value} functions are maximized if the agent is following an \emph{optimal policy} $\Pi^* \in \arg\max_{\tilde{\Pi}} V^{\tilde{\Pi}}$. 
\ifshort\else In general the environment can dispense any real-valued reward. But, for a simplified analysis we assume that rewards are bounded and positive. \fi

\begin{assumption}{$(\mathrm{Bounded \ positive \ reward})$} \label{asm:positive-rewards}
    We assume bounded and positive rewards and without loss of generality we assume $\R :\subseteq [0,1]$. 
\end{assumption}

It is easy to see that the bounded rewards bound the \mbox{(action-)value} functions between zero and $1/(1-\g)$.

\subsection{Homomorphism Setup}

We define a homomorphism as a surjective mapping $\psi$ from the original state-action space $\H \times \A$ to the abstract state-action space $\S \times \B$. 
For a succinct exposition, we also define a few marginalized mapping functions. These marginalized maps do not have any special significance other than making the notation a bit simpler. 

\paradot{Histories mapped to an $sb$-pair} For a given abstract action $b \in \B$, we define a marginalized abstract state map as
\beq
\psi^{-1}_b(s) := \left\lbrace h \in \H \mid \exists a \in \A : \psi(h,a)=(s,b)\right\rbrace.
\eeq

\paradot{Actions mapped to an $sb$-pair} Similarly, we also define a marginalized abstract action map for any abstract state $s \in \S$ and history $h \in \H$ as
\beq
\psi^{-1}_s(b) := \left\lbrace a \in \A \mid \psi(h,a)=(s,b)\right\rbrace.
\eeq

It is important to note that $\psi^{-1}_s(b)$ is also a function of history. This dependence is always clear from the context, so we suppress it in the notation. 

\paradot{Abstract states mapped by a history} By a slight abuse of notation we overload $\psi$, and define a history to abstract state marginalized map as
\beq
\psi(h) := \{s \in \S \mid \exists a \in \A, b \in \B : \psi(h,a)=(s,b)\}.
\eeq

\paradot{Histories mapped to an abstract state} Finally, an abstract state to history marginalized map is defined as
\beq
\psi^{-1}(s) := \{h \in \H \mid \exists a \in \A, b \in \B : \psi(h,a)=(s,b)\}.
\eeq

\ifshort\else In this work, we assume a structure for the aggregation map $\psi$. The general unstructured case is left for future research. \fi

\begin{assumption}{$(\psi(h)=s)$}\label{asm:state-history}
    We assume that an abstract state is determined only by the history --- i.e. $\psi(h,a) := (s=f(h),b)$, where $f$ is any fixed \emph{surjective} function of history and is independent of actions $a$ and $b$. 
\end{assumption}

The above assumption implies that $\psi(h)$ is \emph{singleton}. This is not only a technical necessity but a requirement to make the mapping \emph{causal}, i.e. the current history $h$ corresponds to a unique state $s$ independent of the next action taken by the agent. If we drop this assumption then the current history might resolve to a different state based on the next (future) action taken by the agent.

A homomorphic map $\psi$ lets the agent merge the experiences from $P$ and induces a \emph{history-based} abstract process $P_\psi$. Formally, for all $\psi(h,a) = (s,b)$ and any next abstract state $\ps$, we express $P_\psi$ as    
\beq\label{eq:marginP}
P_\psi(\ps r|ha) := \sum_{\to : \psi(ha\to r)=\ps} P(\to r|ha).
\eeq

The map $\psi$ also induces a \emph{history-based} abstract policy $\Pi_{\psi}$ as
\beq\label{eq:dpih}
\Pi_\psi(b|h) := \sum_{\ta \in \psi^{-1}_{s}(b)} \Pi(\ta|h).
\eeq

It is clear from \eqref{eq:marginP} and \eqref{eq:dpih} that the induced abstract process and policy are in general non-Markovian, i.e. both are functions of the history $h$ and not only the abstract state $s$.

\paradot{Non-MDP homomorphisms} In this work we consider two types of non-Markovian homomorphisms: {\bf a)} \emph{Q-uniform} homomorphisms, where the state-action pairs are merged if they have close Q-values, i.e. $Q^\Pi(h,a) \approx Q^\Pi(\ph,\pa)$ for all $\psi(h,a)=\psi(\ph,\pa)$, and {\bf b)} \emph{V-uniform} homomorphisms, when the merged state-action pairs have close values , i.e. $V^\Pi(h) \approx V^\Pi(\ph)$ for all $\psi(h)=\psi(\ph)$. A formal treatment of these non-MDP homomorphisms is provided in the main results section. In both Q and V-uniform homomorphisms, $P_\psi$ can be history-dependent, in result, the abstract process is non-MDP.

\section{Motivation for Non-MDP Homomorphisms}\label{sec:Example}

In this section we motivate the importance of non-MDP homomorphisms by an example. We show that a non-MDP homomorphism can cater to a large set of domains and allows more compact representations.

\paradot{Navigational Grid-world} Let us consider a simplified version of the asymmetric grid-world example by \citeay{Ravindran2004} in Figure \ref{fig:grid-world-original}. In this navigational domain, the goal of an agent $\Pi$ is to navigate the grid to reach the target cell $T$. The unreachable cells are grayed-out. The agent receives a large positive reward if it enters the cell $T$, otherwise a small negative reward is given to the agent at each time-step. The agent is capable of moving in the four directions, i.e. up, down, left and right. This domain has an \emph{almost} similar transition and reward structure across a diagonal axis. We call this an approximate MDP axis and denote it by $\approx$MDP. This axis of symmetry enables us to create a homomorphism of the domain using approximately half of the original state-space (see Figure \ref{fig:grid-world-abstract}).

\begin{figure}[h]
    \centering
    \resizebox{!}{
    \ifshort\else 7cm \fi
    }
    {%
    \begin{tikzpicture}
    [
    box/.style={rectangle,draw=black,thick, minimum size=1cm},
    gray-box/.style={box, fill=gray}
    ]
    
    \foreach \x in {0,1,...,5}{
        \foreach \y in {0,1,...,5}
        \node[box] at (\x,\y){};
    }
    \node[gray-box] at (2,5){};  
    \node[gray-box] at (4,4){};  
    \node[gray-box] at (2,3){};  
    \node[gray-box] at (3,2){};  
    \node[gray-box] at (0,2){};  
    \node[gray-box] at (5,2){};  
    \node[gray-box] at (2,0){};  
    \node[gray-box] at (5,0){}; 
    
    \node at (5,5){T};
    \draw[->, ultra thick] (1.25,4) -- (1.75,4);
    \node at (1,4){$\Pi$};
    \draw[->, ultra thick] (4,1.25) -- (4, 1.75);
    \node at (4,1){$\Pi$};
    
    \node[above] at (6,6){$\approx$MDP};
    \draw[dashed] (-1,-1) -- (6,6);
    \node at (5.7,5.7) {\AxisRotator[x=0.2cm,y=0.4cm,->,rotate=60]};
    \end{tikzpicture}
    }
    \caption{The original navigational grid-world with the axis of approximate symmetry. The gray cells are not reachable. The target cell is at the top right corner. The figure shows two possible positions of the agent and corresponding optimal actions.}\label{fig:grid-world-original}
\end{figure}
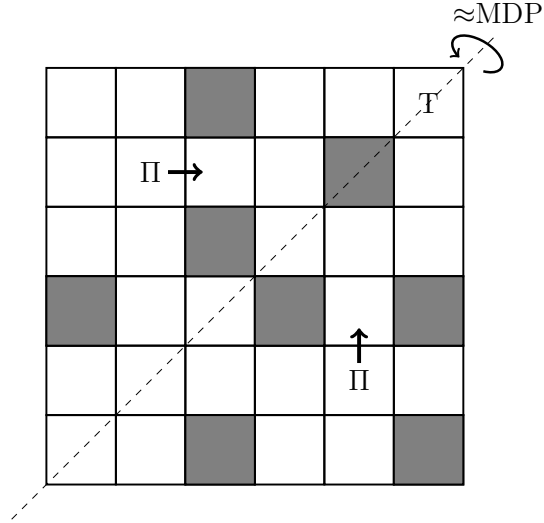
\begin{figure}[h]
    \centering
    \resizebox{!}
    {
        \ifshort\else 5.5cm \fi
    }
    {%
    \begin{tikzpicture}
    [
    box/.style={rectangle,draw=black,thick, minimum size=1cm},
    gray-box/.style={box, fill=gray},
    stripe-box/.style={box, pattern=north west lines, pattern color=gray}
    ]
    
    \foreach \y in {1,...,5}{
        \foreach \x in {0,1,...,\y}
        \node[box] at (\x,\y){};
    }
    \node[box] at (0,0){};
    \node[gray-box] at (2,5){};  
    \node[gray-box] at (4,4){};  
    \node[gray-box] at (2,3){};  
    \node[gray-box] at (0,2){};  
    \node[stripe-box] at (0,5){};
    
    \node at (5,5){T};
    \draw[->, ultra thick] (1.25,4) -- (1.75,4);
    \node at (1,4){$\pi$};
    
    \end{tikzpicture}}
    \caption{A possible MDP homomorphism by merging the mirror state-action pairs together. The presence of a hashed cell indicates that it is not an exact homomorphism. The agent $\pi$ solves the problem in this abstract domain.}\label{fig:grid-world-abstract}
    \hfill
\end{figure}
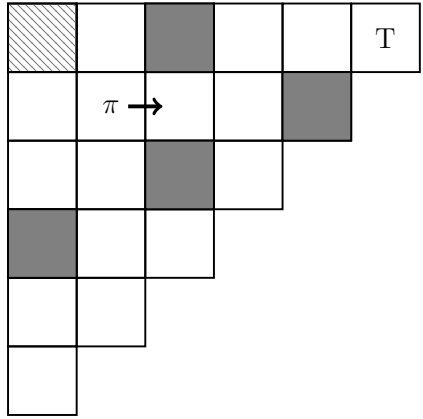

This grid-world example has primarily been studied in the context of either exact, approximate or Bounded parameter MDP (BMDP) homomorphisms \cite{Ravindran2004}: the abstract model \emph{approximately} preserves the one-step dynamics of the original environment. However, as we later prove in this paper (see Theorem \ref{thm:psiQstar}ii), some non-MDP homomorphisms can also be used to find a near-optimal policy in the original process. We motivate the need of non-MDP homomorphisms, first, by highlighting the fact\footnote{This section is an informal motivation, we formally deal with this fact in the main results section (Theorem \ref{thm:psimdpstar}i).} that in the grid-world domain, the states with similar dynamics have similar optimal action-values. Afterwards, we modify the grid-world domain such that the modified grid-world does not have an approximate MDP symmetry axis, but still has the same approximate optimal action-values symmetry.

We apply Value Iteration (VI) \cite{Bellman1957} with some fixed but irrelevant parameters on the grid world (see Figure \ref{fig:optimal-values-gw}). The grid world has the same approximate symmetry axis for the optimal values, denoted by $\approx$Q-uniform axis. It is easy to see that each merged state in Figure \ref{fig:grid-world-abstract} has the same action-values. Hence, the $\approx$MDP axis is also an $\approx$Q-uniform axis in the grid-world.

\begin{figure}
    \centering
    \resizebox{!}
    {
        \ifshort\else 7cm \fi
    }
    {%
    \begin{tikzpicture}
    [
    box/.style={rectangle,draw=black,thick, minimum size=1cm},
    gray-box/.style={box, fill=gray},
    highlight/.style={box,draw=black,ultra thick,dashed}
    ]
    
    \foreach \x in {0,1,...,5}{
        \foreach \y in {0,1,...,5}
        \node[box] at (\x,\y){};
    }
    \node[gray-box] at (2,5){};  
    \node[gray-box] at (4,4){};  
    \node[gray-box] at (2,3){};  
    \node[gray-box] at (3,2){};  
    \node[gray-box] at (0,2){};  
    \node[gray-box] at (5,2){};  
    \node[gray-box] at (2,0){};  
    \node[gray-box] at (5,0){}; 
    
    \node at (0,0){1.74};
    \node at (0,1){1.94};
    \node at (0,3){2.42};
    \node at (0,4){2.70};
    \node at (0,5){\textbf{2.42}};
    
    \node at (1,0){1.94};
    \node at (1,1){2.17};
    \node at (1,2){2.42};
    \node at (1,3){2.70};
    \node at (1,4){3.01};
    \node at (1,5){\textbf{2.17}};
    
    \node at (2,1){2.42};
    \node at (2,2){2.17};
    \node at (2,4){3.35};
    
    \node at (3,0){2.42};
    \node at (3,1){2.70};
    \node at (3,3){3.35};
    \node at (3,4){3.74};
    \node at (3,5){4.16};
    
    \node at (4,0){2.70};
    \node at (4,1){3.01};
    \node at (4,2){3.35};
    \node at (4,3){3.74};
    \node at (4,5){4.64};
    
    \node at (5,1){\textbf{2.70}};
    \node at (5,3){4.16};
    \node at (5,4){4.64};
    \node at (5,5){5.16};
    
    \node[highlight] at (4,2){};
    \node[highlight] at (2,4){};
    
    \node[above] at (6,6){$\approx$Q-uniform};
    \draw[dashed] (-1,-1) -- (6,6);
    \node at (5.7,5.7) {\AxisRotator[x=0.2cm,y=0.4cm,->,rotate=60]};
    
    \end{tikzpicture}}
    \caption{The optimal values at each approachable cell. The bold-faced values are not exactly matched across the symmetry axis.} \label{fig:optimal-values-gw}
\end{figure}

\paradot{Modified Navigational Grid-world} Now we modify the grid world such that it does not have an $\approx$MDP axis (Figure \ref{fig:grid-world-original}) but it still has the same $\approx$Q-uniform axis (Figure \ref{fig:optimal-values-gw}). The idea is to take a pair of merged states from Figure \ref{fig:grid-world-original} and change the reward and transition probabilities such that the states no longer have similar one-step dynamics but still have similar action-values. For example, let us consider the cells highlighted with dashed borders in Figure \ref{fig:optimal-values-gw} and denote the cell in the bottom half with $s_{23}$. Let $u,d,p_u$ and $p_d$ denote the actions up and down, and the probabilities to reach the desired cell by taking the corresponding action, respectively. Let $r_u$ and $r_d$ be the expected rewards for each action in the state $s_{23}$. In general, we get an \emph{under-determined} set of equations for the action-value function at state $s_{23}$ as
\beq\label{eq:q23}
Q^*(s_{23},a) = \begin{cases}
    r_u + 0.73\g p_u  + 3.01\g   &\quad \text{if }a = u\\
    r_d - 0.73\g p_d + 3.74\g  &\quad \text{if } a = d.
\end{cases}
\eeq

In the original navigational grid-world problem $p_u = p_d = 1$, i.e. each action leads deterministically to the indented reachable cell, and $r_u=r_d = r_n$, where $r_n$ is a fixed small negative reward. We can break the $\approx$MDP similarity by setting\footnote{$p_u = 0$ implies that the action $u$ now takes the agent to the down cell and vice versa for the action $d$.} $p_u = p_d := 0$, i.e. the actions behave in the opposite way in the lower half, $r_u := r_n + 0.73\g$ and $r_d := r_n - 0.73\g$, without changing the $\approx$Q-uniform similarity. In fact, we can have infinite combinations of rewards and transitions to get a set of modified domains since the set of equations \eqref{eq:q23} is under-determined.

This set of \emph{modified} domains, by design, no longer allows the approximate MDP homomorphism of Figure \ref{fig:grid-world-abstract}. Every state is different in terms of reward and transition structure across the $\approx$MDP axis of Figure \ref{fig:grid-world-original}. Any one-step model similarity abstraction would be approximately of the same size as the original problem. However, if we consider Q-uniform homomorphisms, i.e. state-action pairs are merged if the action-values are close, then the set of modified domains has a same Q-uniform homomorphism.

In GRL, it is natural to assume that the (expected) rewards are function of realized history. The above modification argument is more likely to hold in a GRL setting: the reward and transition similarity might be hard to satisfy. Therefore, a GRL agent is better to consider such non-MDP homomorphisms to cover more domains with a single abstract model. Now we ask the main question, does such a non-MDP homomorphism, e.g. Q-uniform homomorphism, have a guaranteed solution for the original problem? In the next section, we answer this question in affirmative for Q-uniform homomorphisms (Theorem \ref{thm:psiQstar}ii), but in negative for V-uniform 
homomorphisms with a weaker positive result (Theorem \ref{thm:psiVstar}ii). 

\section{Key Elements to Go Beyond MDPs}\label{sec:key-elements}

This section introduces the key elements of the paper that enables us to prove performance bounds for non-MDP homomorphisms.

\subsection{A Stochastic Inverse and Surrogate MDP}

The key idea to get a near-optimal policy of the true environment $P$ is to transform $P_\psi$ into a surrogate MDP on the abstract state-action space. 
Afterwards, the optimal policy of this surrogate MDP is uplifted to $P$. This technique of casting a non-MDP process as an MDP has been used in ESA \cite{Hutter2016}. To get this surrogate MDP, we define a stochastic inverse $B$ of the homomorphism $\psi$ as a probability measure over the history-action space given an abstract state-action pair, formally, $B : \S \times \B \rightarrow \Delta(\H \times \A)$. Moreover, we require $B(ha|sb) := 0$ for any $\psi(h,a) \neq (s,b)$. The surrogate MDP is defined as
\beq\label{eq:beyondMDP}
p_B(\ps\pr | sb) := \sum_{\th \in \H, \ta \in \A} P_\psi(\ps\pr | \th\ta)B(\th\ta|sb).
\eeq

It might seem like a paradoxical idea to solve a non-Markovian $P_\psi$ using an MDP $p_B$, but the paradox is superficial. It is the stochastic inverse that complements the non-Markovianness of $P_\psi$. Finding such an inverse algorithmically, hence the surrogate MDP, is not a trivial task in general \cite{Hutter2016}. 

This action-dependent stochastic inverse separates our work from the action-independent weighting function considered by \citeay{Abel2016}. Although, learning of such weighting function is beyond the scope of this paper, an action-independent weighting function is not learnable. Because, when this weighting function is built from the true sampling distribution, it becomes action-dependent. \citeay{Hutter2016} constructs such a learnable action-dependent inverse for the state abstraction case. Fortunately, the choice of $B$ becomes irrelevant in Q-uniform homomorphisms (see Theorem \ref{thm:psiQstar}), however this is not the case in V-uniform homomorphisms (see Theorem \ref{thm:psiVstar}). 

Similar to the original process, we also define the \mbox{(action-)value} functions for the surrogate MDP $p_B$ on the abstract state-action space $\S \times \B$ with an abstract state-based policy $\pi$. The action-value function is given as
\beq\label{eq:q}
q^\pi(s,b) := \sum_{\ts \in \S, \tr \in \R} p_B(\ts \tr | s b) \left( \tr + \g v^\pi(\ts)\right) \eeq
where the value function is
\beq\label{eq:v}
v^\pi(s) := \sum_{\tb \in \B} q^\pi(s, \tb) \pi(\tb | s).
\eeq

An abstract state-based \emph{optimal} policy $\pi^*$ is a value maximizer, i.e. $\pi^* \in \arg\max_{\tilde{\pi}}v^{\tilde{\pi}}$.

\subsection{Representative Abstract Policy}

As discussed earlier, we are primarily interested in the optimal policies of the surrogate MDP. However, it is also interesting to consider a general policy case (e.g. Theorems \ref{thm:psimdppi}, \ref{thm:exactmdp}, \ref{thm:psiQpi} and \ref{thm:psiVpi}) akin to an on-policy result where we uplift a representative policy. We use any arbitrary member as a representative policy $\piR$ on the abstract state $s$. 
\beq\label{eq:dpi}
\piR(\cdot|s) := \Pi_{\psi}(\cdot|\th), \quad  \text{for some } \th \in \psi^{-1}(s).
\eeq

This arbitrary choice of representative introduces a policy representation error $\DPi$ for each abstract state $s$, expressed as
\beq\label{eq:Delta}
\DPi(s) := \sup_{\th \in \psi^{-1}(s)} \norm{ \piR(\cdot|s) - \Pi_\psi(\cdot|\th)}_1.
\eeq

This representation error is small/zero when the induced abstract policy $\Pi_{\psi}$ is approximately/piecewise constant, i.e. $\Pi_{\psi}(\cdot|h) = \Pi_\psi(\cdot|\ph)$ for all $\psi(h) = \psi(\ph)$.

\ifshort\else It is easy to see that the state aggregation mapping function $\phi$ of ESA setup\footnote{The reader is encouraged to see \cite{Hutter2016} for more details about $\phi$.} is a special case of our generalized mapping function $\psi$. We can mimic any $\phi$ by using a $\psi_\phi$ with identity transformation over action spaces --- i.e. $\psi_\phi(h,a) := (\phi(h),a)$. \fi

In the next section, we provide the main results of this work. We construct a near-optimal policy for the original process from the surrogate MDP even if the homomorphism is non-MDP. 

\section{Main Results}\label{sec:Aggregation}

\ifshort\else In this section, we prove that under some conditions an optimal policy of the abstract process losses only a faction of the value when uplifted in the original process. Also, these results hold even when the marginalized process in not an MDP. \fi
We analyze three types of homomorphisms in this work: MDP, Q-Uniform and V-Uniform homomorphisms. Both Q and V-Uniform homomorphisms are non-Markovian by definition. In general, MDP and Q-Uniform homomorphisms admit a \emph{deterministic} near-optimal policy of the original process, while V-Uniform homomorphisms do not.

\subsection{Markov Decision Process Homomorphisms}

A homomorphism is an MDP homomorphism if the induced abstract process $P_\psi$ is an MDP. Then, there exists a process $p_{\mathrm{MDP}}$ such that for all $\psi(h,a) = (s,b)$ and for all $\ts$ and $\tr$, it holds:
\beq\label{eq:MDP}
P_\psi(\ts\tr|ha) = p_{\mathrm{MDP}}(\ts\tr|sb).
\eeq

Using the above condition in \eqref{eq:beyondMDP}, renders $p_B = p_{\mathrm{MDP}}$ and independent of $B$.
The condition $\eqref{eq:MDP}$ is a  \emph{stronger} version of the bisimulation condition \cite{Givan2003} that is generalized to joint history-action pairs. This condition is strong enough to preserve the optimal \mbox{(action-)value} functions of the original process (see Theorem \ref{thm:psimdpstar}). But, it is not strong enough to preserve arbitrary policy \mbox{(action-)value} functions (see Theorem \ref{thm:psimdppi}). Unless we define a notion of action-value function representative and a corresponding representation error. For an abstract state-action pair, the representative action-value is defined as
\beq\label{eq:Qrep}
Q^\Pi(\psi^{-1}(s,b)) := Q^\Pi(\th, \ta), \,  \text{for some } \psi(\th, \ta) = (s,b)
\eeq
and the representation error of the action-value function is expressed as
\beq\label{eq:delta}
\DQ(s) := \sup_{\th,\ta, \tb:\psi(\th,\ta) = (s,\tb)} \abs{ Q^\Pi(\psi^{-1}(s,\tb)) - Q^\Pi(\th, \ta)}.
\eeq

Similar to $\DPi$, this representation error is small/zero if the action-value function is approximately/piecewise constant.
At this point, we have all the required components properly defined to state the first theorem of the paper.
\begin{theorem}{$(\psi_{\mathrm{MDP}\Pi})$}\label{thm:psimdppi} 
    Let $\psi$ be a homomorphism such that $P_\psi$ is an MDP, then for any policy $\Pi$ and all $\psi(h,a) = (s,b)$ it holds:
    \bqan
    \left|q^\piR(s,b) - Q^\Pi(h,a)\right| &\leq&  \frac{\g \eps_{\max}}{1-\g} \quad \text{and} \quad \\
    \left|v^\piR(s) - V^\Pi(h)\right| &\leq&  \frac{\eps_{\max}}{1-\g}
    \eqan
    where $\eps_{\max} := \max_{\ts \in \S} \left(\DQ(\ts) + \frac{\DPi(\ts)}{1-\g}\right)$.
\end{theorem}
\ifshort\else
\begin{proof}
   See Appendix~\ref{sec:proofs}.
\end{proof}
\fi

The above theorem shows that the surrogate MDP \emph{approximately} preserves the \mbox{(action-)value} functions of the original process for any arbitrary policy. However, these \mbox{(action-)value} functions are preserved \emph{exactly} if we further impose a policy uniformity condition in addition to an MDP assumption. 

\begin{theorem}{$(\psi_{\mathrm{MDP}\Pi=})$} \label{thm:exactmdp}
    Let $\psi$ be a homomorphism such that $P_\psi$ is an MDP and $\Pi(\cdot|h) = \Pi(\cdot|\ph)$ (i.e. the policy similarity condition holds) for some policy $\Pi$ and for all $\psi(h) = \psi(\ph)$. Then for all $\psi(h,a) = (s,b)$ it holds:
    \beqn
    q^\piR(s,b) = Q^\Pi(h,a) \quad \text{and} \quad v^\piR(s) = V^\Pi(h).
    \eeqn
\end{theorem}
\ifshort\else
\begin{proof}
    See Appendix~\ref{sec:proofs}.
\end{proof}
\fi

Theorems \ref{thm:psimdppi} and \ref{thm:exactmdp} are important but not very useful results. As already discussed, we are interested in the \mbox{(near-)optimal} policies of the original process. And, we want to find the abstract policies that can be lifted with a performance guarantee from the abstract state-action space to the original history-action space.

\begin{theorem}{$(\psi_{\mathrm{MDP}*})$}\label{thm:psimdpstar} 
    Let $\psi$ be a homomorphism such that $P_\psi$ is an MDP, then for all $\psi(h,a)=(s,b)$ it holds:
    \begin{enumerate}[(i)]
        \item $q^*(s,b) = Q^*(h,a)$ and  $v^*(s) = V^*(h)$.
        \item $V^*(h)=V^{\breve{\Pi}}(h)$ 
        
        where $\breve{\Pi}(h) :\in \psi^{-1}_s\left(\pi^*(s)\right)$ for any $\psi(h)=s$.
    \end{enumerate}
\end{theorem}
\ifshort\else
\begin{proof}
    See Appendix~\ref{sec:proofs}.
\end{proof}
\fi

For any MDP homomorphism, the performance guarantee is provided by Theorem \ref{thm:psimdpstar}ii. The abstract optimal policy $\pi^*$ is also an optimal policy for the original process when lifted to the original history-action space.

\subsection{Q-Uniform Homomorphisms}\label{sec:Extreme}

In this section, we relax the MDP condition (see Equation \ref{eq:MDP}) on the abstract-process provided by the homomorphism. We show that there still exists an abstract policy that is near-optimal in the original process (see Theorem \ref{thm:psiQstar}ii). We start with proving a value loss bound for an arbitrary policy when the action-value function of the original process is \emph{approximately} $\psi$-uniform.

\begin{theorem}{$(\psi_{Q^\Pi})$}\label{thm:psiQpi}
    Assume $\left|Q^\Pi(h,a)-Q^\Pi(\ph,\pa)\right|\leq \eps$ for some policy $\Pi$ and for all $\psi(h,a)=\psi(\ph,\pa)$. Then for all $\psi(h,a) = (s,b)$ it holds:
    \bqan
    \left|Q^\Pi(h,a) - q^\piR(s,b)\right| &\leq&  \eps + \frac{\g\eps(s)}{1-\g} \quad \text{and} \quad \\ \left|V^\Pi(h) - v^\piR(s)\right| &\leq& \frac{\eps(s)}{1-\g}
    \eqan
    where $\eps(s) := 2\eps + \frac{\DPi(s)}{1-\g}$.
\end{theorem}
\ifshort\else
\begin{proof}
    See Appendix~\ref{sec:proofs}.
\end{proof}
\fi

The following theorem \emph{improves} the optimal policy value loss bounds, \emph{cf.} Theorem \ref{thm:psiQpi}, and establishes the existence of a near-optimal policy of the original history-action space in the abstract state-action space.  

\begin{theorem}{$(\psi_{Q^*})$}\label{thm:psiQstar}
    Let $|Q^*(h,a)-Q^*(\ph,\pa)|\leq \eps$ for all $\psi(h,a)=\psi(\ph,\pa)$, then for all $\psi(h,a) = (s,b)$ it holds:
    \begin{enumerate}[(i)]
        \item
        $|Q^*(h,a) - q^*(s,b)| \leq \frac{2\eps}{1-\g}$ and
        
        $|V^*(h) - v^*(s)| \leq \frac{2\eps}{1-\g}$.
        \item
        $0 \leq V^*(h) - V^{\breve{\Pi}}(h) \leq \frac{4\eps}{(1-\g)^2}$ 
        
        where $\breve{\Pi}(h) :\in \psi^{-1}_s\left(\pi^*(s)\right)$ for any $\psi(h)=s$.
    \end{enumerate}
\end{theorem}
\ifshort\else
\begin{proof}
    See Appendix~\ref{sec:proofs}.
\end{proof}
\fi

It is important to note that Theorem \ref{thm:psiQstar} holds for any stochastic inverse $B$. Every choice of $B$ gives a different surrogate MDP $p_B$, so the theorem provides a \emph{near-optimal} performance guarantee for the \emph{uplifted} abstract optimal policies of \emph{any} possible surrogate MDP. Therefore, for any non-MDP Q-uniform homomorphism and a fixed $B$ there exists an uplifted near-optimal policy ($\breve{\Pi}$ from Theorem \ref{thm:psiQstar}ii).

\subsection{V-Uniform Homomorphisms}

All the previous results are valid for any choice of the stochastic inverse $B$. However, for V-uniform homomorphisms, the results are explicitly dependent on $B$ (see Theorem \ref{thm:psiVpi} and \ref{thm:psiVstar}). We need a couple of more entities to express the results of this section. We denote the $B$-average of the action-value function of the original process as
\beq
\langle Q^\Pi(\psi^{-1}(s,b)) \rangle_B := \sum_{\th \in \H, \ta \in \A} Q^\Pi(\th,\ta)B(\th\ta|sb).
\eeq

Furthermore, we can decompose $B$ into two distinct parts: action dependent and independent.  With an abuse of notation, assume an arbitrary joint distribution $B$ over $\H,\A,\S$ and $\B$. By using the chain rule of probability distributions on $B$,
\bqa
B(ha|sb) 
&=& B(h|sb)B(a|bhs) \nonumber \\
&=& \frac{B(hs)B(b|hs)}{B(sb)}B(a|bhs) \nonumber \\
&\overset{(a)}{=}& \frac{B(hs)B(b|h)}{B(sb)}B(a|bh) \nonumber \\
&=& B(h|s)\frac{B(b|h)}{B(b|s)}B(a|bh) \nonumber \\
&=& \underbrace{B(h|s)}_{\text{action-independent}}
\cdot\overbrace{\left(\frac{B(ab|h)}{B(b|s)}\right)}^{\text{action-dependent}} \label{eq:decomposed-B}
\eqa
$(a)$ follows from Assumption \ref{asm:state-history}, the state is determined only by the history.

Using the action-dependent part from \eqref{eq:decomposed-B}, we define a history and state based induced measure on the original action space for any $B$ and an abstract state based policy $\pi$ as
\beq
B^\pi(a|hs) := \sum_{\tb \in \B}\left(\frac{B(a\tb|h)}{B(\tb|s)}\right) \pi(\tb | s).
\eeq

This seemingly complex and arbitrary relationship has a well-structured explanation. If \emph{approximately}, the $B$ distribution is linked to the actual dynamics of an agent $\pi$ acting in the abstract state-action space, i.e. $B(b|s) \approx \pi(b|s)$, then $B^\pi(a|hs) \approx B(a|h)$, which is effectively a \emph{shadow} agent induced by the agent $\pi$ on the original history-action space.

To prove a result analogous to Theorem \ref{thm:psiQpi} for a V-uniform homomorphism, we need to impose an extra condition on $B$, \emph{ cf.} Theorem \ref{thm:psiQpi}, which requires a structure on $B$ and/or on the underlying original process. For general $B$, there exist some known counter examples \cite{Hutter2016}.

\begin{theorem}{$(\psi_{V^\Pi})$}\label{thm:psiVpi}
    Let $\abs{V^\Pi(h)-V^\Pi(\ph)}\leq \eps$ for some policy $\Pi$ and for all $\psi(h)=\psi(\ph)$, and $\abs{\sum_{\ta \in \A} Q^\Pi(h,\ta)B^\piR(\ta|hs) - V^\Pi(h)} \leq \eps_B$ for all $s=\psi(h)$, then it holds:
    \bqan
    \left|\langle Q^\Pi(\psi^{-1}(s,b))\rangle_B - q^\piR(s,b)\right| &\leq& \frac{\g(\eps+ \eps_B)}{1-\g} \quad \text{ and} \quad \\
    \left|V^\Pi(h) - v^\piR(s)\right| &\leq& \frac{\eps + \eps_B}{1-\g}.
    \eqan
\end{theorem}
\ifshort\else
\begin{proof}
    See Appendix~\ref{sec:proofs}.
\end{proof}
\fi

In Theorem \ref{thm:psiQpi}, we had an absolute loss-bound for action-value functions but in Theorem \ref{thm:psiVpi} we only have a $B$-average relationship. So far, we were able to get a near-optimal performance guarantee when the optimal policy of a surrogate MDP is uplifted to the original process (see Theorems \ref{thm:psimdpstar}ii and \ref{thm:psiQstar}ii). However, there does not exist such a near-optimal performance guarantee for V-uniform homomorphisms. A \emph{deterministic} abstract policy could be arbitrarily worse off when uplifted to the original process \cite[Theorem 10]{Hutter2016} in V-uniform state-abstractions, which are a special case of V-uniform homomorphisms. However, a relatively weak result is still possible.

\begin{theorem}{$(\psi_{V^*})$}\label{thm:psiVstar}
    Let $|V^*(h)-V^*(\ph)|\leq \eps$ for all $\psi(h)=\psi(\ph)$ and $|\sum_{\ta \in \A} Q^*(h,\ta)B^{\pi^*}(\ta|hs) - V^*(h)| \leq \eps_B$ for all $s=\psi(h)$, then for all $\psi(h,a) = (s,b)$ it holds:
    \begin{enumerate}[(i)]
        \item $|\langle Q^*(\psi^{-1}(s,b))\rangle_B - q^*(s,b)| \leq \frac{3\g(\eps + \eps_B)}{(1-\g)^2}$ and 
        
        $|V^*(h) - v^*(s)| \leq \frac{3(\eps + \eps_B)}{(1-\g)^2}$.
        \item If $\eps + \eps_B = 0$ then $\psi(h, \Pi^*(h)) = (s, \pi^*(s))$ for all $\psi(h)=s$. 
    \end{enumerate}
\end{theorem}
\ifshort\else
\begin{proof}
    See Appendix~\ref{sec:proofs}.
\end{proof}
\fi

In the \emph{approximate} case, i.e. $\eps + \eps_B > 0$, Theorem \ref{thm:psiVstar} is not as useful as Theorem \ref{thm:psiQstar} because of the missing performance guarantee, {\em cf.} Theorem \ref{thm:psiQstar}ii. However, it is still an important theorem for the \emph{exact} V-uniform homomorphisms, i.e. $\eps + \eps_B = 0$. In that case, it is guaranteed that the optimal actions of \emph{all} member histories are mapped to the \emph{same} abstract optimal action (see Theorem \ref{thm:psiVstar}ii).

\section{Discussion, Outlook and Conclusion}\label{sec:Disc}

In this paper we analyzed \emph{approximate} homomorphisms of a general history-based environment. The main idea was to find a \emph{deterministic} policy in the abstract state-action space such that, when uplifted, it is a near-optimal policy in the original problem. Using the surrogate MDP technique, we proved near-optimal performance bounds for both MDP (Theorem \ref{thm:psimdpstar}ii) and Q-uniform homomorphisms (Theorem \ref{thm:psiQstar}ii). In general, there does not exist a near-optimal \emph{deterministic} uplifted policy for V-uniform homomorphisms. However, we proved a weaker result (Theorem \ref{thm:psiVstar}ii) for the \emph{exact} V-uniform homomorphisms: the optimal actions of the member histories are mapped to the same abstract optimal action at the corresponding state of the surrogate MDP. 

\paradot{Versus ESA}
We borrow some notation and techniques from \citeay{Hutter2016}. But this work is crucially different from ESA. Apart from the obvious difference of being a generalization to homomorphisms, there are also some other key differences. In ESA, the policy $\Pi$ is required to be state uniform for various of the main results \cite[Theorems 1,5,6 and 9]{Hutter2016}, whereas we do not make any such assumption.
\ifshort\else Moreover, at the first instance our results might look \emph{almost} similar to ESA but the important difference is in the definition of $\eps(s)$ which is \emph{not} a simple addition of both state and action representation errors. It is a non-trivial weighted average of representation errors. \fi
The extra conditions on Theorems \ref{thm:psiVpi} and \ref{thm:psiVstar} are \emph{weaker} than the policy-uniformity condition, \emph{cf.} \cite[Theorems 6 and 9]{Hutter2016}, and do not have direct counterparts in ESA.

\paradot{Versus Options} As briefly addressed in the introduction section, the options framework does not have any provable performance guarantees, yet. Whereas our restriction of uplifting a state-based policy and using a deceptively ?spatial-looking? abstraction of actions have such guarantees. Since we allow the action mapping part of $\psi$ to be a function of history, which is arguably a function of time, our framework also admits temporal dependencies. It enables $\psi$ to model much more than mere renaming of the original action space distributions. A thorough comparison between these two approaches is left for future work. 

\subsection{Outlook}
The results in this work are quite general but there are various open questions left for future research. 

\paradot{Reinforcement Learning (Learning Problem)} For a given homomorphism $\psi$, the most obvious question we left open is the choice of $B$. We call this the \emph{learning problem}. Two of the three main results in this work (Theorems \ref{thm:psimdpstar}ii and \ref{thm:psiQstar}ii) are valid for any choice of $B$, so any fixed $B$ would suffice. But the third main result (Theorem \ref{thm:psiVstar}ii) is very much involved with the choice of $B$. However, it is not a strong result in itself. Nevertheless, in a state-abstraction context, $B$ facilitates learning of the surrogate MDP from the induced abstract process \cite{Hutter2016}. Therefore, it is an intriguing direction to explore for homomorphisms. 

\paradot{Feature Reinforcement Learning (Discovery Problem)} The focus of this paper is to provide performance guarantees for a given homomorphism. While in practice, the agent has to learn such a reduction/model from experience. It is known as the \emph{discovery problem} \cite{Li2006} in RL and \emph{Feature Reinforcement Learning} (FRL) \cite{Hutter2009} in the GRL context. It is non-trivial to solve this problem even in a state-abstraction framework \cite{Hutter2016}. Our result can help to build such an FRL algorithm for homomorphisms, e.g. during the model learning/building, the algorithm may use the bounds from this work to select/discard a candidate model. 

\paradot{Special Environment Classes} In general, we do not use/exploit structure of the underlying original process. However, effects of a specific model class can be expressed in terms of the \mbox{(action-)value} functions. For example, if the original process is a finite state POMDP then our results provide the performance-loss guarantee by representing a belief-state based value function of the POMDP by a state-based value function. A similar argument can be rendered for various other types of model classes. Since the results in this work are general, they are not expected to gracefully scale down to some class specific tight performance bounds. Nevertheless, it is an important agenda to get the scaled-down variants of these results for some specific model classes.

\paradot{Continuous state-action space} The results in this paper easily extend to the continuous state-action space homomorphisms for the \emph{measurable} maps. The summations change to integrals and the measurability constraint make sure that these integrals are well-defined. In this case, a homomorphism map has a natural interpretation of being a discretization of the underlying space. However, it is sometimes desirable to use a restricted continuity condition, e.g. Lipschitz or Holder continuity, rather than the weak measurability constraint. A continuous state-action homomorphism under some restricted continuity constraints would be a nice generalization of our results.

\paradot{Fully Generalized Homomorphism} In a sense our results are not \emph{fully} general since we assumed a structure on the homomorphism. A fully generalized homomorphism formulation with no $\psi(h,a) = (f(h),b)$ assumption would be an interesting extension of this work. However, lifting this condition may lead to some bizarre non-causal effects, e.g. the \emph{current} abstract state would be decided by the \emph{next} action!

\subsection{Conclusion}
In conclusion, our results show that in GRL the near-optimal performance guarantee is not limited only to MDP homomorphisms. It is sometimes possible to have non-MDP models, i.e. Q-uniform homomorphisms, with bounded performance loss guarantees.
\ifshort\else 
We also relax the \emph{strong} policy-uniformity condition in ESA to allow stochastic policies. \fi

\iffinal
\paradot{Acknowledgements}
\todo{} \ifshort We thank {\em Elliot Catt} and {\em Amy Zhang} for
proofreading the earlier drafts and the anonymous reviewers for their valuable feedbacks.\fi This work has in parts been supported by Australian Research Council grant 
DP150104590.
\fi


\ifconf
\bibliographystyle{aaai} 
\else
\bibliographystyle{plainnat} 
\fi

\begin{small}
    
    \bibliography{../library}
    
\end{small}

\ifshort\else
\newpage
\clearpage
\appendix
\onecolumn

\section{Omitted Proofs}\label{sec:proofs}
We provide all omitted proofs in this appendix. We use $\lessgtr$ to denote a two side inequality, e.g. if $\abs{x-y} \leq \eps$, which in effect implies  $-\eps \leq x - y \leq \eps$, we  denote it as $x \lessgtr y \pm \eps$ to express both inequalities at the same time.

\subsection{Proof of Theorem \ref{thm:psimdppi}}

\begin{proof}
    Let $\d := \underset{\th,\ta,\ts,\tb: \psi(\th,\ta)=(\ts,\tb)}{\sup}\left|q^\pi(\ts,\tb) - Q^\Pi(\th,\ta)\right|$, and for any $\psi(h)=s$ we have,
    \bqa\label{eq:vVdif}
    \left|v^\pi(s) - V^\Pi(h)\right| 
    &\overset{(a)}{\leq}& \left|\sum_{\tb \in \B} q^\pi(s, \tb)\pi(\tb|s) - \sum_{\tb \in \B} Q^\Pi(\psi^{-1}(s, \tb))\pi(\tb|s)\right| + \DQ(s) + \frac{\DPi(s)}{1-\g} \nonumber \\
    &\overset{}{\leq}& \sum_{\tb \in \B} \left|q^\pi(s, \tb) - Q^\Pi(\psi^{-1}(s, \tb))\right| \pi(\tb|s) + \DQ(s) + \frac{\DPi(s)}{1-\g} \nonumber \\
    & \leq & \d + \DQ(s) + \frac{\DPi(s)}{1-\g}
    \eqa
    $(a)$ follows from the definition of $v^\pi(s)$ and Lemma \ref{lem:Vrep}. Now for any $\psi(h,a)=(s,b)$, we have,
    \bqan
    Q^\Pi(h,a) &\overset{}{\equiv}& \sum_{\to \in \O, \tr \in \R} P(\to \tr |ha)(\tr + \g V^\Pi(\th))  \qquad [\th = ha\to\tr] \\
    &\overset{(b)}{\lessgtr}& \sum_{\ts \in \S, \tr \in \R} P_\psi(\ts \tr |ha)\left(\tr + \g v^\pi(\ts)\right) \pm \g\left(\d + \DQ(\ts) + \frac{\DPi(\ts)}{1-\g}\right) \\
    &\overset{(\ref{eq:MDP})}{=}& \sum_{\ts \in \S, \tr \in \R} p(\ts \tr |sb)\left(\tr + \g v^\pi(\ts)\right) \pm \g\left(\d + \DQ(\ts) + \frac{\DPi(\ts)}{1-\g}\right) \\ 
    &\overset{}{\lessgtr}& q^\pi(s,b) \pm \g(\d + \eps_{\max}) 
    \eqan
    $(b)$ follows from value function error bound (\ref{eq:vVdif}) and definition of $P_\psi$ given by (\ref{eq:marginP}).
    Since, $\d \equiv \sup \abs{q^\pi(s,b) - Q^\Pi(h,a)} \leq \g(\d + \eps_{\max})$ therefore, $\d \leq \frac{\g \eps_{\max}}{1-\g}$, hence completes the proof. 
\end{proof}

\subsection{Proof of Theorem \ref{thm:psimdpstar}}

\begin{proof}
    {\bf(i)}
    Let $\d := \underset{h,a,s,b: \psi(h,a)=(s,b)}{\sup} \abs{q^*(s,b) - Q^*(h,a)}$. Now for any $\psi(h)=s$ we have,
    \bqa\label{eq:vVdifstar}
    \left|v^*(s) - V^*(h)\right| 
    &\overset{(a)}{\leq}& \left|\max_{\tb \in \B} q^*(s, \tb) - \max_{\tb \in \B} Q^*(\psi^{-1}(s, \tb))\right| + \DQ(s)  \nonumber \\
    &\overset{(b)}{\leq}& \d
    \eqa
    $(a)$ follows from the definitions of $v^*(s)$ and Lemma \ref{lem:VStarrep}, and $(b)$ is due to Lemma \ref{lem:dzeromdp}.
    \bqan
    Q^*(h,a) &\overset{}{\equiv}& \sum_{\to \in \O, \tr \in \R} P(\to \tr |ha)(\tr + \g V^*(\th)) \qquad [\th = ha\to\tr] \nonumber \\
    &\overset{(\ref{eq:vVdifstar})}{\lessgtr}& \sum_{\ts \in \S, \tr \in \R} P_\psi(\ts \tr |ha)(\tr + \g v^*(\ts)) \pm \g\d \nonumber \\
    &\overset{(\ref{eq:MDP})}{=}& \sum_{\ts \in \S, \tr \in \R} p(\ts \tr |sb)(\tr + \g v^*(\ts))  \pm \g\d \nonumber \\ 
    &\overset{}{\equiv}& q^*(s,b) \pm \g\d
    \eqan
    Therefore, $\d \leq \g\d$, therefore $\d = 0$ which completes the proof.
    
    {\bf(ii)} 
    For $\psi(h)=s$ and $\breve{\Pi}(h) :\in \psi^{-1}_s(\pi^*(s))$,
    \bqan
    V^*(h) &\overset{(i)}{=}& v^*(s) \overset{}{\equiv} q^*(s,\pi^*(s))  \overset{(i)}{=} Q^*\left(h, \breve{\Pi}(h)\right)
    \eqan
    which implies $Q^*\left(h, \breve{\Pi}(h)\right) = V^*(h)$ and Lemma \ref{lem:different-optimal-action} concludes the proof.
    
\end{proof}

\subsection{Proof of Theorem \ref{thm:exactmdp}}

\begin{proof}
    Let $\d := \underset{h,\ph: \psi(h)=\psi(\ph), a, \pa}{\sup} |Q^\Pi(h,a) - Q^\Pi(\ph,\pa)|$, then for all $\psi(h) = \psi(\ph)$,
    \bqa\label{eq:vdifmdp}
    \left|V^\Pi(h) - V^\Pi(\ph)\right| 
    &\overset{}{\equiv}& \left|\sum_{\ta \in \A} Q^\Pi(h,\ta) \Pi(\ta|h) - \sum_{\ta \in \A} Q^\Pi(\ph,\ta)\Pi(\ta|\ph) \right| \nonumber \\
    &\overset{(a)}{=}& \left|\sum_{\ta \in \A} \left(Q^\Pi(h,\ta)-Q^\Pi(\ph,\ta)\right) \Pi(\ta|h) \right| \leq \d
    \eqa
    $(a)$ follows from the assumption. Now for all $\psi(h,a) = \psi(\ph,\pa)=(s,b)$,
    \bqa\label{eq:qdifmdp}
    \left|Q^\Pi(h,a) - Q^\Pi(\ph,\pa)\right| 
    &\overset{(b)}{=}& \abs{\sum_{\ts \in \S, \tr \in \R} P_\psi(\ts\tr | ha) \left(\tr + \g V^\Pi(\th)\right)  \right. \nonumber\\&&\left.-  \sum_{\ts \in \S, \tr \in \R} P_\psi(\ts\tr | \ph\pa) \left(\tr + \g V^\Pi(\tilde{\ph})\right) } \nonumber \\
    &\overset{(\ref{eq:MDP})}{=}& \g\left|\sum_{\ts \in \S, \tr \in \R} p(\ts\tr | sb) \left( V^\Pi(\th) - V^\Pi(\tilde{\ph})\right)\right| \nonumber \\ &\overset{(\ref{eq:vdifmdp})}{\leq}& \g\d
    \eqa
    $(b)$ follows from the definition and $\th = ha\to\tr$ and $\tilde{\ph} = \ph\pa\to\tr$, and $\psi(\th) = \psi(\tilde{\ph}) = \ts$. From the inequality (\ref{eq:qdifmdp}), we have $\d \leq \g \d \Rightarrow \d = 0$. Therefore, $Q^\Pi(h,a) = Q^\Pi(\ph,\pa)$ for all $\psi(h,a) = \psi(\ph,\pa)$. Note that this also implies, $\DQ = 0$ and $\DPi = 0$ by assumption.
\end{proof}

\subsection{Proof of Theorem \ref{thm:psiQpi}}
\begin{proof}
    Let $\d := \underset{\th,\ta,\ts,\tb:\psi(\th,\ta)=(\ts,\tb)}{\sup}\left|Q^\Pi(\th,\ta) - q^\pi(\ts,\tb)\right|$, and for any $\psi(h)=s$ we have,
    \bqa
    V^\Pi(h) - v^\pi(s) 
    &\overset{Lem. \ref{lem:Vrep}}{\lessgtr}& \sum_{\tb \in \B} \left( Q^\Pi(\psi^{-1}(s, \tb))\pi(\tb|s) - q^\pi(s, \tb)\pi(\tb|s) \right) \pm \left(\DQ(s) + \frac{\DPi(s)}{1-\g}\right) \nonumber \\
    &\overset{}{=}& \sum_{\tb \in \B} \left( Q^\Pi(\psi^{-1}(s, \tb)) - q^\pi(s, \tb)\right) \pi(\tb|s) \pm \left(\DQ(s) + \frac{\DPi(s)}{1-\g}\right) \nonumber \\
    &\overset{(a)}{\lessgtr}& \pm \left(\d + \DQ(s) + \frac{\DPi(s)}{1-\g}\right) \label{eq:VvDiff}
    \eqa
    $(a)$ follows from the definition of $\d$ and the fact that $\pi(\tb|s)$-average is smaller than the $\tb$-supremum. Using the above inequality (\ref{eq:VvDiff}) and Lemma \ref{lem:QBqDiff} we get,
    \beq \label{eq:BAvgqDiff}
    | \langle Q^\Pi(\psi^{-1}(s,b)) \rangle_B - q^\pi(s,b) | \leq \g\left(\d + \DQ(s) + \frac{\DPi(s)}{1-\g}\right).
    \eeq
    
    We exploit the theorem's assumption and derive a key relationship between the $B$ average and any instance of action value.
    
    \bqa \label{eq:BAvgRepDiff}
    \langle Q^\Pi(\psi^{-1}(s,b)) \rangle_B 
    &\overset{}{\equiv}& \sum_{\th \in \H, \ta \in \A}   Q^\Pi(\th,\ta)B(\th\ta|sb) \nonumber \\
    &\overset{(a)}{\lessgtr}& \sum_{\th \in \H, \ta \in \A}  \left( Q^\Pi(\psi^{-1}(s,b)) \pm \eps \right) B(\th\ta|sb) \nonumber \\
    &\overset{}{=}& Q^\Pi(\psi^{-1}(s,b)) \pm \eps 
    \eqa
    $(a)$ follows from the theorem's assumption. Since, $Q^\Pi(\psi^{-1}(s,b))$ is a representative member in the pre-image set of $(s,b)$; it is equivalent to say $Q^\Pi(\psi^{-1}(s,b)) = Q^\Pi(h,a)$ for any $\psi(h,a) = (s,b)$. Therefore, combining \eqref{eq:BAvgqDiff} and \eqref{eq:BAvgqDiff} we get $\left|Q^\Pi(h,a) - q^\pi(s,b)\right| \leq \g(\d + \DQ(s) + \frac{\DPi(s)}{1-\g}) + \eps$, hence $\d \leq \frac{\eps + \g \DQ(s) + \frac{\g\DPi(s)}{1-\g}}{1-\g}$.
\end{proof}

\subsection{Proof of Theorem \ref{thm:psiVpi}}
\begin{proof}
    Let $\d := \underset{\ts,\th: \psi(\th)=\ts}{\sup}\left|V^\Pi(\th)-v^\pi(\ts)\right|$ then for any $\psi(h)=s$, we have,
    \bqa
    \sum_{\tb \in \B} \langle Q^\Pi(\psi^{-1}(s,\tb))\rangle_B \pi(\tb|s) 
    &\overset{(a)}{=}& \sum_{\th \in \H} B_\H(\th|s)  \sum_{\ta \in \A} Q^\Pi(\th,\ta) \sum_{\tb \in \B} B_\A(\ta\tb|s\th) \pi(\tb|s) \nonumber \\
    &\overset{(b)}{\lessgtr}& \sum_{\th: \psi(\th) = s} B_\H(\th|s)\left(V^\Pi(\th) \pm \eps_B \right) \nonumber \\
    &\overset{(c)}{\lessgtr}& \sum_{\th: \psi(\th) = s} B_\H(\th|s)\left(V^\Pi(h) \pm (\eps + \eps_B) \right) \nonumber \\
    &\overset{}{=}& V^\Pi(h) \pm (\eps + \eps_B) \label{eq:QAvgV} 
    \eqa
    $(a)$ follows from the chain rule of joint distributions and since the history is action independently mapped; $(b)$ and $(c)$ follow from the theorem's assumptions. Further, we have,
    \bqa 
    \left|\sum_{\tb \in \B}\langle Q^\Pi(\psi^{-1}(s,\tb))\rangle_B\pi(\tb|s) - v^\pi(s)\right|
    &\overset{}{\equiv}& \left|\sum_{\tb \in \B}\left(\langle Q^\Pi(\psi^{-1}(s,\tb))\rangle_B - q^\pi(s,\tb)\right)\pi(\tb|s) \right| \nonumber \\ 
    &\overset{(d)}{\leq}& \sum_{\tb \in \B}\left|\langle Q^\Pi(\psi^{-1}(s,\tb))\rangle_B - q^\pi(s,\tb)\right|\pi(\tb|s) \nonumber \\ 
    &\overset{(e)}{\leq}& \g\d \label{eq:QAvgq}
    \eqa
    $(d)$ follows from the simple mathematical fact that $|\sum_x f(x)| \leq \sum_x |f(x)|$; $(e)$ uses Lemma \ref{lem:QBqDiff}. Now we prove the main result. Together with (\ref{eq:QAvgV}) and (\ref{eq:QAvgq}) we have,
    \bqan
    \left|V^\Pi(h) - v^\pi(s)\right| & \leq & 
    \left|V^\Pi(h) - \sum_{\tb \in \B} \langle Q^\Pi(\psi^{-1}(s,\tb))\rangle_B\pi(\tb|s) \right| \\&&+ \left|\sum_{\tb \in \B} \langle Q^\Pi(\psi^{-1}(s,\tb))\rangle_B\pi(\tb|s) - v^\pi(s) \right| \\
    &\leq & \g\d + \eps + \eps_B
    \eqan
    Hence, $\d \leq \frac{\eps + \eps_B}{1 - \g}$ and completes the proof.
\end{proof}

\subsection{Proof of Theorem \ref{thm:psiQstar}}

\begin{proof}
    {\bf(i)} The proof follows the same steps as the proof of Theorem \ref{thm:psiQpi}, replacing $\Pi$ with $\Pi^*$ and $\pi$ with $\pi^*$ and using Lemma \ref{lem:VStarrep} instead of Lemma \ref{lem:Vrep}. In the end we use Lemma \ref{lem:eQPi} to conclude the proof.
    
    
    {\bf(ii)} For $\psi(h)=s$ and $\breve{\Pi}(h) :\in \psi^{-1}_s(\pi^*(s))$,
    \bqan
    V^*(h) \pm \frac{2\eps}{1-\g} 
    &\overset{(i)}{\lessgtr}& v^*(s) \overset{}{\equiv} q^*(s,\pi^*(s))  \overset{(i)}{\lessgtr} Q^*\left(h, \breve{\Pi}(h)\right) \pm \frac{2\eps}{1-\g}
    \eqan
    which implies $\abs{Q^*\left(h, \breve{\Pi}(h)\right) - V^*(h)} \leq \frac{4\eps}{1-\g}$ and Lemma \ref{lem:different-optimal-action} concludes the proof.
\end{proof}

\subsection{Proof of Theorem \ref{thm:psiVstar}}

\begin{proof}
    Let us define $\pi_h(s)$ such that $(s, \pi_h(s)) := \psi(h,\Pi^*(h))$ for $\psi(h)=s$. Then,
    \beq
    q^{\pi_h}\left(s,\pi_h(s)\right) = v^{\pi_h}(s) \overset{(a)}{\lessgtr} V^*(h) \pm \frac{\eps + \eps_B}{1-\g} \label{eq:qVe1}
    \eeq
    $(a)$ follows from Theorem \ref{thm:psiVpi} applied to $\Pi = \Pi^*$ (with $\pi = \pi_h$). Now we derive a bound for any $b \in \B$.
    \bqa \label{eq:qVe2}
    q^{\pi_h}(s,b) - \frac{\g(\eps + \eps_B)}{1-\g} 
    &\overset{Thm. \ref{thm:psiVpi}}{\leq}& \langle Q^*\left(\psi^{-1}(s,b)\right) \rangle_B \equiv \sum_{\th \in \H, \ta \in \A} Q^*(\th,\ta) B(\th\ta| sb) \nonumber \\
    &\overset{}{=}& \sum_{\th \in \H} B_\H(\th|s) \sum_{\ta \in \A} Q^*(\th,\ta) B_\A(\ta | sb\th) \nonumber \\
    &\overset{(b)}{\leq}& \sum_{\th \in \H} B_\H(\th|s) \sum_{\ta \in \A} Q^*(\th,\Pi^*(\th)) B_\A(\ta | sb\th) \nonumber \\
    &\overset{}{=}& \sum_{\th: \psi(\th)=s} B_\H(\th|s) V^*(\th) \nonumber \\
    &\overset{(c)}{\leq}& V^*(h) + \eps
    \eqa
    $(b)$ is due to the definition of optimal value and $(c)$ follows form the theorem's assumptions. Together (\ref{eq:qVe1}) and (\ref{eq:qVe2}) imply,
    \beq \label{eq:qVe3}
    v^{\pi_h}(s) = q^{\pi_h}\left(s, \pi_h(s)\right) \overset{}{\leq} \max_{\tb \in \B} q^{\pi_h}(s,\tb) \overset{(\ref{eq:qVe2})}{\leq} V^*(h) + \frac{\eps + \g\eps_B}{1-\g} \overset{(\ref{eq:qVe1})}{\leq} v^{\pi_h}(s) + \frac{2(\eps + \eps_B)}{1-\g} 
    \eeq
    
    \paragraph{(ii)} For $\eps = \eps_B = 0$, the previous equation implies $v^{\pi_h}(s) = \max_{\tb \in \B} q^{\pi_h}(s, \tb)$. It shows that $\pi_h(s) = \pi^*(s)$ for all $\psi(h)=s$.
    
    \paragraph{(i)} Now for general $\eps + \eps_B > 0$ case. For all $s \in \S$ and $b \in \B$ we have,
    \beqn
    0 \overset{}{\leq} q^*(s,b) - q^{\pi_h}(s,b) \overset{}{\equiv} \sum_{\ts \in \S, \tr \in \R} p(\ts\tr|sb) \g\left(v^*(\ts) - v^{\pi_h}(\ts)\right) \overset{(d)}{\leq} \g\max_{\ts \in \S}\left(v^*(\ts) - v^{\pi_h}(\ts)\right)
    \eeqn
    \bqan
    0 \overset{}{\leq} v^*(s) - v^{\pi_h}(s) &\overset{(e)}{\leq}& \max_{\tb \in \B} q^*(s,\tb) - \max_{\tb \in \B} q^{\pi_h}(s,\tb) + \frac{2(\eps + \eps_B)}{1-\g} \\&\overset{(f)}{\leq}& \max_{\tb \in \B}\left(q^*(s,\tb) - q^{\pi_h}(s,\tb)\right)+ \frac{2(\eps + \eps_B)}{1-\g}
    \eqan
    $(d)$ expectation is replace by maximum operation; $(e)$ follows from the definition of $v^*(s)$ and (\ref{eq:qVe3}); $(f)$ is a simple mathematical fact of maximization operation. Together this implies,
    
    \bqa
    \max_{\ts \in \S} \left( v^*(\ts) - v^{\pi_h}(\ts) \right) &\leq & \g\max_{\ts \in \S} \left( v^*(\ts) - v^{\pi_h}(\ts) \right) + \frac{2(\eps + \eps_B)}{1-\g} \nonumber \\
    \Rightarrow \quad \max_{\ts \in \S} \left( v^*(\ts) - v^{\pi_h}(\ts) \right) &\leq & \frac{2(\eps + \eps_B)}{(1-\g)^2} \label{eq:qVe4} 
    \eqa
    Hence for any $\psi(h)=s$, we have,
    \bqan
    V^*(h) - \frac{\eps + \eps_B}{1-\g} &\overset{(\ref{eq:qVe1})}{\leq} & v^{\pi_h}(s) \\
    &\overset{(g)}{\leq}& v^*(s) \\
    &\overset{(\ref{eq:qVe4})}{\leq}& v^{\pi_h}(s) + \frac{2(\eps + \eps_B)}{(1-\g)^2} \\
    &\overset{(\ref{eq:qVe1})}{\leq}& V^*(h) + \frac{\eps + \eps_B}{1-\g} + \frac{2(\eps + \eps_B)}{(1-\g)^2} \leq V^*(h) + \frac{3(\eps + \eps_B)}{(1-\g)^2}
    \eqan
    $(g)$ holds by definition. Hence, the main results follows by using Lemma \ref{lem:QBqDiff} with $\Pi$ replaced by $\Pi^*$ and $\pi$ by $\pi^*$.
\end{proof}

\section{Lemmas}\label{sec:lemmas}

In this section, we establish a couple of important lemmas that bound the (optimal) value loss when we evaluate the action-value function only at the representatives.

\begin{lemma}\label{lem:Vrep}
    For any policy $\Pi$, and $\psi(h)=s$ the following holds,
    \beqn
    \left|V^\Pi(h) - \sum_{\tb \in \B} Q^\Pi(\psi^{-1}(s,\tb)) \piR(\tb|s)\right| \leq \DQ(s) + \frac{\DPi(s)}{1-\g}.
    \eeqn 
\end{lemma}


\begin{proof}
    For any $\psi(h)=s$, we start from the value function of the original process,
    \bqan
    V^\Pi(h) &\overset{}{\equiv}& \sum_{\ta \in \A} Q^\Pi(h, \ta) \Pi(\ta | h) \\
    &\overset{(a)}{=}& \sum_{\tb \in \B} \sum_{\ta \in \psi^{-1}_s(\tb)} Q^\Pi(h, \ta) \Pi(\ta | h) \\
    &\overset{}{=}& \sum_{\tb \in \B} \sum_{\ta \in \psi^{-1}_s(\tb)} \left( Q^\Pi(h,\ta) + Q^\Pi(\psi^{-1}(s, \tb)) - Q^\Pi(\psi^{-1}(s, \tb))  \right)  \Pi(\ta | h) \\
    &\overset{}{=}& \sum_{\tb \in \B} Q^\Pi(\psi^{-1}(s, \tb)) \sum_{\ta \in \psi^{-1}_s(\tb)} \Pi(\ta | h) + \sum_{\tb \in \B} \sum_{\ta \in \psi^{-1}_s(\tb)} \left( Q^\Pi(h,\ta) - Q^\Pi(\psi^{-1}(s, \tb)) \right)  \Pi(\ta | h) \\
    &\overset{(b)}{\leq}& \sum_{\tb \in \B} Q^\Pi(\psi^{-1}(s, \tb)) \Pi_\psi(\tb | h) + \DQ(s) \\
    &\overset{}{=}& \sum_{\tb \in \B} Q^\Pi(\psi^{-1}(s, \tb)) \left(\Pi_\psi(\tb | h) + \piR(\tb |s) - \piR(\tb|s) \right) + \DQ(s) \\
    &\overset{}{=}& \sum_{\tb \in \B} Q^\Pi(\psi^{-1}(s, \tb)) \piR(\tb | s) + \sum_{\tb \in \B} Q^\Pi(\psi^{-1}(s, \tb)) \left(\Pi_\psi(\tb | h) - \piR(\tb |s) \right) + \DQ(s) \\
    &\overset{(c)}{\leq}& \sum_{\tb \in \B} Q^\Pi(\psi^{-1}(s, \tb))\piR(\tb|s) + \frac{\DPi(s)}{1-\g} + \DQ(s) \\
    \eqan
    $(a)$ follows from the fact that the mapping is defined to be surjective; $(b)$ from \eqref{eq:dpih} and \eqref{eq:delta}, $(c)$ uses \eqref{eq:Delta} and the fact that action-value function is bounded, so is the representative. It is easy to get the other side of the inequality from similar steps.
\end{proof}

Now, we bound the error of evaluating the action-value function only at the representatives of the mapping when the agent is following the optimal policy in the original process.

\begin{lemma}\label{lem:VStarrep}
    For the optimal policy $\Pi^*$, and for all $\psi(h)=s$ the following holds,
    \beqn
    \left|V^*(h) - \max_{\tb \in \B} Q^*(\psi^{-1}(s, \tb))\right| \leq \DQ(s).
    \eeqn 
\end{lemma}
\begin{proof}
    For any $\psi(h)=s$ we have,
    \bqan
    V^*(h) &\overset{}{\equiv}& \max_{\ta \in \A} Q^*(h, \ta) \overset{(a)}{=} \max_{\tb \in \B} \max_{\ta \in \psi^{-1}_s(\tb)} Q^*(h, \ta)\\
    &\overset{(\ref{eq:delta})}{\lessgtr}& \max_{\tb \in \B} Q^*(\psi^{-1}(s,\tb)) \pm \DQ(s)
    \eqan
    $(a)$ follows from the fact that mapping is defined to be surjective.
\end{proof}



\begin{lemma}\label{lem:dzeromdp}
    Let $P_\psi$ be an MDP then $Q^*(h,a) = Q^*(\ph,\pa)$ for all $\psi(h,a) = \psi(\ph,\pa)$.
\end{lemma}
\begin{proof}
    Let $\d := \underset{h,\ph, : \psi(h)=\psi(\ph), a, \pa}{\sup} |Q^*(h,a) - Q^*(\ph,\pa)|$, then for all $\psi(h) = \psi(\ph)$,
    \bqa\label{eq:vdifstarmdp}
    \left|V^*(h) - V^*(\ph)\right| &\overset{}{\equiv}& \left|\max_{\ta \in \A} Q^*(h,\ta) - \max_{\ta \in \A} Q^*(\ph,\ta)\right| \nonumber \\
    &\overset{(a)}{\leq}& \max_{\ta \in \A} \abs{Q^*(h,\ta) - Q^*(\ph,\ta)} \leq \d
    \eqa
    $(a)$ follows from simple mathematical fact of maximum value. Now for all $\psi(h,a) = \psi(\ph,\pa) = (s,b)$,
    \bqa\label{eq:qdifstarmdp}
    \left|Q^*(h,a) - Q^*(\ph,\pa)\right| 
    &\overset{(b)}{=}& \abs{\sum_{\ts \in \S, \tr \in \R} P_\psi(\ts\tr | ha) (\tr + \g V^*(\th)) \right.\nonumber\\&&\left.- \sum_{\ts \in \S, \tr \in \R} P_\psi(\ts\tr | \ph\pa) (\tr + \g V^*(\tilde{\ph})) } \nonumber \\
    &\overset{(\ref{eq:MDP})}{=}& \g \left|\sum_{\ts \in \S, \tr \in \R} p(\ts\tr | sb) \left( V^*(\th) - V^*(\tilde{\ph})\right) \right| \nonumber \\ &\overset{(\ref{eq:vdifstarmdp})}{\leq}& \g\d
    \eqa
    $(b)$ follows from the definition and $\th = ha\to\tr$ and $\tilde{\ph} = \ph\pa\to\tr$, and $\psi(\th) = \psi(\tilde{\ph}) = \ts$. By (\ref{eq:qdifstarmdp}) we have $\d \leq \g \d \Rightarrow \d = 0$. 
\end{proof}

The following lemma is a stepping stone for the main results. It establishes a link between the value functions loss and the action-value functions loss.

\begin{lemma}{$(\langle Q \rangle_B - q)$}\label{lem:QBqDiff}
    Let $|V^\Pi(h) - v^\pi(s)| \leq \eps$ for all $\psi(h) = s$. Then for all $s \in \S$ and $b \in \B$ it holds:
    \beqn
    \left| \langle Q^\Pi(\psi^{-1}(s,b)) \rangle_B - q^\pi(s,b) \right| \leq \g\eps.
    \eeqn
\end{lemma}
\begin{proof}
    We begin with the action value of the representative of any $(s,b)$ as
    \bqan
    \langle Q^\Pi(\psi^{-1}(s,b)) \rangle_B \
    &\overset{}{\equiv}& \sum_{\th \in \H, \ta \in \A} B(\th\ta|sb) \sum_{\to \in \O, \tr \in \R} P(\to \tr |\th\ta)\left(\tr + \g V^\Pi(\th\ta\to\tr)\right) \\
    &\overset{(a)}{=}& \sum_{\th \in \H, \ta \in \A} B(\th\ta|sb) \sum_{\ts \in \S} \sum_{\tr \in \R, \to: \psi(\th\ta\to\tr)=\ts} P(\to \tr |\th\ta)\left(\tr + \g V^\Pi(\th\ta\to\tr)\right) \\
    &\overset{(b)}{\lessgtr}& \sum_{\th \in \H, \ta \in \A} B(\th\ta|sb) \sum_{\ts \in \S, \tr \in \R} P_\psi(\ts \tr |\th\ta)\left(\tr + \g v^\pi(\ts) \pm \g\eps \right) \\
    &\overset{}{\equiv}& \sum_{\ts \in \S, \tr \in \R} p_B(\ts \tr |sb)\left(\tr + \g v^\pi(\ts) \right)  \pm \g\eps \equiv q^\pi(s,b) \pm \g\eps
    \eqan
    $(a)$ follows since $\psi$ is surjective; $(b)$ follows from the assumption.
\end{proof}


\begin{lemma}{$(\eps_{Q\Pi})$}\label{lem:eQPi}
    Let $|Q^\Pi(h,a)-Q^\Pi(\ph,\pa)|\leq \eps$ for all $\psi(h,a)=\psi(\ph,\pa)$. Then $\DQ(s) \leq \eps$ for all $s \in \S$.
\end{lemma}
\begin{proof}
    For all $s \in \S$,
    \bqan
    \DQ(s) 
    &\overset{}{\equiv}& \sup_{\th, \ta, \tb : \psi(\th,\ta) = (s,\tb)} | Q^\Pi(\psi^{-1}(s,\tb)) - Q^\Pi(\th, \ta)| \overset{(a)}{\leq} \eps 
    \eqan
    $(a)$ follows from the assumption and the fact that $Q^\Pi(\psi^{-1}(s,b))$ can be any member in the pre-image set of $(s,b)$.
\end{proof}

Moreover, we need Lemma 8 from \cite{Hutter2016}. For completeness, we state the lemma in this work without repeating the proof. It gives value loss-bounds between the optimal policy and a policy that has bounded one-step action-value loss.

\begin{lemma}{$($\cite{Hutter2016} Lemma 8. $\Pi(h) \neq \Pi^*(h))$} \label{lem:different-optimal-action}
    If $Q^*(h,\Pi(h)) \geq V^*(h) - \eps$ for all $h$ and some policy $\Pi$, then for all $h$ and $a$ it follows:
    \bqan
    0 \leq Q^*(h,a) - Q^\Pi(h,a) \leq \frac{\g\eps}{1-\g} \quad \text{and} \quad
    0 \leq V^*(h) - V^\Pi(h) \leq \frac{\eps}{1-\g}.
    \eqan
\end{lemma}

\section{List of Notation}\label{sec:notation}

\begin{tabbing}
    \hspace{0.15\textwidth} \= \hspace{0.73\textwidth} \= \kill
    {\bf Symbol }      \> {\bf Explanation}                                                    \\[0.5ex]
    
    $\SetR$            \> set of real numbers                                                  \\[0.5ex]
    $\times$        \> Cartesian product				                                    \\[0.5ex]
    
    $\Delta(X)$        \> probability distribution over a set $X$                                   \\[0.5ex]
    $X^T$        \> transpose of matrix $X$                                   \\[0.5ex]
    $\tilde{x}$        \> local variable   				                                    \\[0.5ex]
    $x^\prime$         \> different member of a set		                                    \\[0.5ex]
    
    $\P$               \> original process			                                            \\[0.5ex]  
    $\A, \O, \R$       \> continuous action, observation and reward spaces of the original process \\[0.5ex]
    $a, o, r$        \> action, observation and reward tuple $\in \A \times \O \times \R$				                                    \\[0.5ex]
    $\H$               \> set of all histories		                                            \\[0.5ex]  
    $h$        \> history $\in \H$				                                    \\[0.5ex]
    
    $Q, V, \Pi$        \> \mbox{(action-)value} function and policy of the original process               \\[0.5ex]  
    $\Pi^*$        \> optimal policy of the original process				                                    \\[0.5ex]
    $\breve{\Pi}$        \> elevated policy on the original process from the abstract-process			                                    \\[0.5ex]
    
    $\g$               \> discount factor				                                    \\[0.5ex]
    
    $\psi$        \> homomorphism map				                                    \\[0.5ex]
    $P_\psi$        \> marginalized abstract-process				                                    \\[0.5ex]
    $p_B$        \> surrogate MDP process				                                    \\[0.5ex]
    $p_{\text{MDP}}$        \> abstract MDP process				                                    \\[0.5ex]
    $\S, \B$           \> finite state and action spaces of the abstract-process             \\[0.5ex]
    $s,b$        \> state action pair $\in \S \times \B$				                                    \\[0.5ex]
    
    $q,v,\pi$        \> \mbox{(action-)value} function and policy of the surrogate MDP process	         \\[0.5ex]
    $\pi^*$        \> optimal policy of the surrogate MDP process				                                    \\[0.5ex]
    $\Pi_\psi$        \> history-dependent abstract policy				                                    \\[0.5ex]
    $\psi^{-1}_b(s)$        \> set of histories mapped to $(s,b)$ pair				                                    \\[0.5ex]
    $\psi_s(b)$        \> history-dependent set of (original) actions mapped to an $(s,b)$ pair			                                    \\[0.5ex]
    $Q^\Pi(\psi^{-1}(s,b))$        \> action value representative of an $(s,b)$ partition				                                    \\[0.5ex]
    $\DPi$        \> maximum variation among abstracted policy members			                                    \\[0.5ex]
    $\DQ$        \> maximum variation among abstracted action value members				                                    \\[0.5ex]
    $B$        \> stochastic inverse of the homomorphism map				                                    \\[0.5ex]
    $B^\pi$        \> $B$ and $\pi$ induced measure on the original action space				                                    \\[0.5ex]
    $\langle \cdot \rangle_B$        \> $B$ average				                                    \\[0.5ex]
    $\eps, \eps_B, \eps_{\max}$        \> small positive error constants				                                    \\[0.5ex]
\end{tabbing}

\section{Example Aggregations}\label{sec:Example-2}

In this section, we provide a toy example to illustrate the possibility of approximately joint state-action aggregation beyond MDPs. In the example, contrary to ESA, we also aggregate the approximately similar policy state-action pairs.  These pairs are not aggregated in ESA due to the exact policy similarity condition, \emph{cf.} Theorem 5 of \cite{Hutter2016}.

For simplicity, we assume that actions and histories are mapped independently and the original environment $P$ is an MDP. We define the original action and observation spaces as $\A = \O = \{1,2,3,4\}$. Moreover, $\S = \{X, Y\}$, $\B = \{\alpha, \beta\}$. The original state-action pairs are represented by dots and the shaded regions indicate the mapping function. The dynamics are (fictitious) region constant (see Figure \ref{fig:example})\footnote{Only the dots are elements of the joint space and the regions are fictitious. The aggregated pair $(s,b)$ is indicated adjacent to each region.}. 

We assume an arbitrary policy which may depend on the history rather than only on the last observation. This allows complex dynamics for the proof of concept. We express the environment as
\bqan
P(\po|oa) &=& \sum_{\ta \in \A} P(\po|oa) \Pi(\ta|oa\po) \\
&=& \sum_{\ta \in \A} P^\Pi(\po\ta|oa).
\eqan

We express the dynamics with a joint measure $P^\Pi$ and do not distinguish between the policy and the environment unless otherwise stated. Let the rewards be a function of the originating region and the problem has a finite set of real-valued rewards. We present three cases in this example: non-MDP dynamics, approximate action-values and policy disagreement.

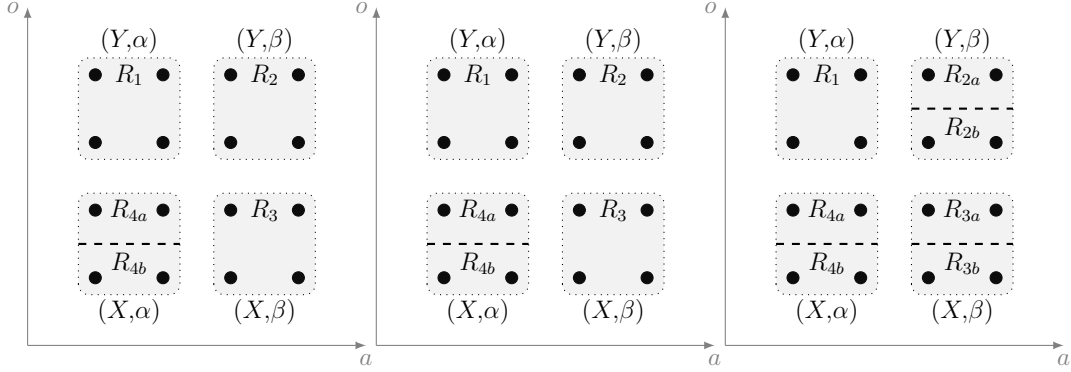
\begin{figure*}
    \begin{subfigure}[b]{0.30\textwidth}
        \centering
        \begin{tikzpicture}[scale=0.45, every node/.style={scale=0.8}]
        \coordinate (Origin)   at (0,0);
        \coordinate (XAxisMin) at (0,0);
        \coordinate (XAxisMax) at (10,0);
        \coordinate (YAxisMin) at (0,0);
        \coordinate (YAxisMax) at (0,10);
        \draw [thin, gray,-latex] (XAxisMin) -- (XAxisMax) node [below] {$a$};
        \draw [thin, gray,-latex] (YAxisMin) -- (YAxisMax) node [left] {$o$};
        \foreach \x in {1,2,...,4}{
            \foreach \y in {1,2,...,4}{
                \node[draw,circle,inner sep=2pt,fill] at (2*\x,2*\y) {};
            }
        }
        \filldraw[fill=gray, fill opacity=0.1, draw=black, dotted, rounded corners] (1.5,1.5)
        rectangle (4.5,4.5);
        \filldraw[fill=gray, fill opacity=0.1, draw=black, dotted, rounded corners] (5.5,1.5)
        rectangle (8.5,4.5);
        \filldraw[fill=gray, fill opacity=0.1, draw=black, dotted, rounded corners] (1.5,5.5)
        rectangle (4.5,8.5);
        \filldraw[fill=gray, fill opacity=0.1, draw=black, dotted, rounded corners] (5.5,5.5)
        rectangle (8.5,8.5);
        
        \draw[dashed, thick] (1.5,3) -- (4.5,3); 
        
        \node at (3,8) {$R_1$};
        \node at (7,8) {$R_2$};
        \node at (7,4) {$R_3$};
        \node at (3,4) {$R_{4a}$};
        \node at (3,3) [below] {$R_{4b}$};

        \node at (3,1) {$(X,\alpha)$};
        \node at (7,1) {$(X,\beta)$};
        \node at (3,9) {$(Y,\alpha)$};
        \node at (7,9) {$(Y,\beta)$};
        \end{tikzpicture} 
    \end{subfigure}
    \begin{subfigure}[b]{0.30\textwidth}
        \centering
        \begin{tikzpicture}[scale=0.45, every node/.style={scale=0.8}]
        \coordinate (Origin)   at (0,0);
        \coordinate (XAxisMin) at (0,0);
        \coordinate (XAxisMax) at (10,0);
        \coordinate (YAxisMin) at (0,0);
        \coordinate (YAxisMax) at (0,10);
        \draw [thin, gray,-latex] (XAxisMin) -- (XAxisMax) node [below] {$a$};
        \draw [thin, gray,-latex] (YAxisMin) -- (YAxisMax) node [left] {$o$};
        \foreach \x in {1,2,...,4}{
            \foreach \y in {1,2,...,4}{
                \node[draw,circle,inner sep=2pt,fill] at (2*\x,2*\y) {};
            }
        }
        \filldraw[fill=gray, fill opacity=0.1, draw=black, dotted, rounded corners] (1.5,1.5)
        rectangle (4.5,4.5);
        \filldraw[fill=gray, fill opacity=0.1, draw=black, dotted, rounded corners] (5.5,1.5)
        rectangle (8.5,4.5);
        \filldraw[fill=gray, fill opacity=0.1, draw=black, dotted, rounded corners] (1.5,5.5)
        rectangle (4.5,8.5);
        \filldraw[fill=gray, fill opacity=0.1, draw=black, dotted, rounded corners] (5.5,5.5)
        rectangle (8.5,8.5);
        
        \draw[dashed, thick] (1.5,3) -- (4.5,3); 
        
        \node at (3,8) {$R_1$};
        \node at (7,8) {$R_2$};
        \node at (7,4) {$R_3$};
        \node at (3,4) {$R_{4a}$};
        \node at (3,3) [below] {$R_{4b}$};

        \node at (3,1) {$(X,\alpha)$};
        \node at (7,1) {$(X,\beta)$};
        \node at (3,9) {$(Y,\alpha)$};
        \node at (7,9) {$(Y,\beta)$};
        
        \end{tikzpicture}
    \end{subfigure}
    \begin{subfigure}[b]{0.30\textwidth}
        \centering
        \begin{tikzpicture}[scale=0.45, every node/.style={scale=0.8}]
        \coordinate (Origin)   at (0,0);
        \coordinate (XAxisMin) at (0,0);
        \coordinate (XAxisMax) at (10,0);
        \coordinate (YAxisMin) at (0,0);
        \coordinate (YAxisMax) at (0,10);
        \draw [thin, gray,-latex] (XAxisMin) -- (XAxisMax) node [below] {$a$};
        \draw [thin, gray,-latex] (YAxisMin) -- (YAxisMax) node [left] {$o$};
        \foreach \x in {1,2,...,4}{
            \foreach \y in {1,2,...,4}{
                \node[draw,circle,inner sep=2pt,fill] at (2*\x,2*\y) {};
            }
        }
        \filldraw[fill=gray, fill opacity=0.1, draw=black, dotted, rounded corners] (1.5,1.5)
        rectangle (4.5,4.5);
        \filldraw[fill=gray, fill opacity=0.1, draw=black, dotted, rounded corners] (5.5,1.5)
        rectangle (8.5,4.5);
        \filldraw[fill=gray, fill opacity=0.1, draw=black, dotted, rounded corners] (1.5,5.5)
        rectangle (4.5,8.5);
        \filldraw[fill=gray, fill opacity=0.1, draw=black, dotted, rounded corners] (5.5,5.5)
        rectangle (8.5,8.5);
        
        \draw[dashed, thick] (1.5,3) -- (4.5,3); 
        \draw[dashed, thick] (5.5,3) -- (8.5,3); 
        \draw[dashed, thick] (5.5,7) -- (8.5,7); 
        
        \node at (3,8) {$R_1$};
        \node at (7,8) {$R_{2a}$};
        \node at (7,7) [below] {$R_{2b}$};
        \node at (7,4) {$R_{3a}$};
        \node at (7,3) [below] {$R_{3b}$};
        \node at (3,4) {$R_{4a}$};
        \node at (3,3) [below] {$R_{4b}$};

        \node at (3,1) {$(X,\alpha)$};
        \node at (7,1) {$(X,\beta)$};
        \node at (3,9) {$(Y,\alpha)$};
        \node at (7,9) {$(Y,\beta)$};
        \end{tikzpicture}
    \end{subfigure}
    \caption{\textbf{(Left):} Non-MDP aggregation, \textbf{(Middle):} Approximate aggregation, \textbf{(Right):} Violating policy uniformity condition}
    \label{fig:example}
\end{figure*}

\paradot{Non-MDP Example} In the first case, this example (see Figure \ref{fig:example}:Left) demonstrates a non-MDP aggregation. Let the problem have the following region uniform transition probability matrix in the observation-action space:
\beq
P^\Pi =
\begin{bmatrix}
    0 & 1 & 0 & 0  & 0 \\
    0 & 0 & 1 & 0  & 0 \\
    0 & 0 & 0 & 1/2 & 1/2 \\
    1 & 0 & 0 & 0  & 0 \\
    1/2 & 0 & 0 & 1/4 & 1/4
\end{bmatrix}
\eeq
where $P^\Pi_{ij}$ is the probability of reaching region $R_j$ if the current $(o,a) \in R_i$\footnote{The indexes should be read in order --- i.e. $1,2,3,4a,4b$.}. Formally, $P^\Pi_{ij} = \sum_{(\to,\ta) \in R_j} P^\Pi(\to\ta|oa \in R_j)$. The marginalized process is expressed as
\beqn
P_\psi(X|ha) = P_\psi(X|oa) = \sum_{\to : \psi(oa\to)=X, \ta \in \A} P^\Pi(\to\ta|oa).
\eeqn

This is not even approximately an MDP. It is evident from the following probabilities of reaching state $X$ from itself.
\bqan
P_\psi(X|oa \in R_{4a}) =& P^\Pi_{4a3} + P^\Pi_{4a4a} + P^\Pi_{4a4b} &= 0 \\
P_\psi(X|oa \in R_{4b}) =& P^\Pi_{4b3} + P^\Pi_{4b4a} + P^\Pi_{4b4b} &= \frac{1}{2}
\eqan

The above equations show that reaching state $X$ from the regions $R_{4a}$ and $R_{4b}$ are different. But, we can still aggregate the regions if they have similar action-values.
Let the regional rewards are $r = \begin{bmatrix} 0 & 0 & 0 & \g & 0 \end{bmatrix}$. The regional action values can be expressed as
\beq
Q^\Pi(o,a) = r(oa) + \g \sum_{\to \in \O,\ta \in \A} Q^\Pi(\to,\ta)P^\Pi(\to\ta|oa).
\eeq

It translates into a vectorization form with each region $i$ has the action value expressed as
\beq
Q_i = r_i + \g \sum_{j} P^\Pi_{ij}Q_j.
\eeq
By solving this system of equations we get the following region uniform action-value vector,
\beq
Q = \begin{bmatrix}
    c-2 & \g^2c & \g c & c & c
\end{bmatrix}
\eeq
where $c = \frac{2}{1-\g^3} \geq 2$. Hence, the example shows that even the regions $R_{4a}$ and $R_{4b}$ have non-MDP dynamics, they can still be aggregated due to the same action-values.

\paradot{Approximate Q-value Example} Now for the second case, we perform an approximate aggregation in region $R_3$ (see Figure \ref{fig:example}:Middle). Let the reward in $R_{3b}$ is $\eps$. The updated reward vector is $r = \begin{bmatrix} 0 & 0 & 0 & \eps & \g & 0 \end{bmatrix}$. By keeping everything same as in the first example, the new transition matrix is given as
\beq
P^\Pi =
\begin{bmatrix}
    0 & 1 & 0 & 0  & 0 & 0 \\
    0 & 0 & 1/2 & 1/2  & 0 & 0\\
    0 & 0 & 0 & 0 & 1/2 & 1/2 \\
    0 & 0 & 0 & 0 & 1/2 & 1/2 \\
    1 & 0 & 0 & 0  & 0 & 0 \\
    1/2 & 0 & 0 & 0 & 1/4 & 1/4
\end{bmatrix}
\eeq

Similar to the first case, we get the aggregated action-values as
\beq
Q = \begin{bmatrix}
    c-2 & \frac{\g\eps}{2}+ \g^2c & \g c & \g c + \eps & c & c
\end{bmatrix}
\eeq
where $c = \frac{\g^2\eps + 4}{2(1-\g^3)} \geq 2$. It shows that the regions $R_{3a}$ and $R_{3b}$ can be approximately aggregated together.

\paradot{Approximate Policy Example} As the last case, let us divide the region $R_2$ into two regions with different policies (see Figure \ref{fig:example}:Right). Let the policy is approximately similar, i.e. $\abs{\Pi(a|o \in R_{2a}) - \Pi(a|o \in R_{2b})} = \eps^\prime$ for all $\psi(o,a) = (Y,\beta)$. It makes region $R_1$  approachable from itself because $\eps^\prime$ weight of the aggregated action $\beta$ in region $R_{2b}$ is distributed to the aggregated action $\alpha$ in region $R_{1}$. This effectively translates into the following region uniform transition matrix,
\beqn
P^\Pi =
\begin{bmatrix}
    \eps^\prime & 1/2 & 1/2-\eps^\prime & 0 & 0  & 0 & 0 \\
    0 & 0 & 0 & 1/2 & 1/2  & 0 & 0\\
    0 & 0 & 0 & 1/2 & 1/2  & 0 & 0\\
    0 & 0 & 0 & 0 & 0 & 1/2 & 1/2 \\
    0 & 0 & 0 & 0 & 0 & 1/2 & 1/2 \\
    1 & 0 & 0 & 0 & 0  & 0 & 0 \\
    1/2 & 0 & 0 & 0 & 0 & 1/4 & 1/4
\end{bmatrix}
\eeqn

The reward structure is the same as in the previous case, $r = \begin{bmatrix} 0 & 0 & 0 & 0 & \eps & \g & 0 \end{bmatrix}$. Finally, we get the action-value vector as
\beq
Q = \begin{bmatrix}
    \g^2 g(\eps^\prime)(\eps/2 + \g c) \\ \frac{\g\eps}{2} + \g^2 c \\ \frac{\g\eps}{2} + \g^2 c \\ \g c \\ \g c + \eps \\ c \\ c
\end{bmatrix}^T
\eeq
where $c = \frac{4 + \g^2\eps g(\eps^\prime)}{2(1-\g^3g(\eps^\prime))}$ and $g(\eps^\prime) := \frac{1-\eps^\prime}{1-\g\eps^\prime}$. This final case shows that although the regions $R_{2a}$ and $R_{2b}$ have different policies they still have same action-values. Hence, the regions can be aggregated together exactly. It allows us to have coarser maps than ESA permits.

\ifuc

\section{An Example Q-uniform Homomorphism}\label{sec:oldExample}

In this section we provide an example Q-uniform homomorphism. The example is a powerful template homomorphism that can be used for a set of original environments. The abstract state space is $\S := \{X, Y\}$ and the abstract action space is $\B := \{\alpha, \beta\}$. The original action $\A$ and observation $\O$ spaces are finite. Let $p_Y(ha) := P_\psi(Y|ha)$ is the probability that the abstract state $Y$ is reached from a history $h$ after taking action $a$. We also assume that the environment has an extra structure such that $p_Y(h):= p_Y(ha) = p_Y(h\pa)$, i.e. the current action does not affect the next abstract state transition but only the history does. For any history $h$ and action $a$, let the reward function satisfies the following structure:

\beq\label{eq:example-reward}
    r(ha) := \begin{cases}
    -\g p_Y(ha) \quad &\text{if } \psi(h,a) = (X,\alpha) \\
    -\frac{1}{2}-\g p_Y(ha) \quad &\text{if } \psi(h,a) = (X,\beta) \\
    \frac{1}{2}-\g p_Y(ha) \quad &\text{if } \psi(h,a) = (Y,\alpha) \\
    1-\g p_Y(ha) \quad &\text{if } \psi(h,a) = (Y,\beta)
\end{cases}
\eeq

It is easy to see that for any history $h$ and action $a$ the optimal action-value function is given as,

\beq\label{eq:example-q}
Q^*(h,a) = r(ha) + \g \left(\sum_{\po\pr:\psi(ha\po\pr) =X}P(\po\pr|ha)V^*(ha\po\pr) + \sum_{\po\pr:\psi(ha\po\pr) =Y}P(\po\pr|ha)V^*(ha\po\pr)\right)
\eeq

From \eqref{eq:example-reward} and \eqref{eq:example-q} we get a $\psi$ uniform action value function,

\beq
Q^*(h,a) := \begin{cases}
    0 \quad &\text{if } \psi(h,a) = (X,\alpha) \\
    -\frac{1}{2} \quad &\text{if } \psi(h,a) = (X,\beta) \\
    \frac{1}{2} \quad &\text{if } \psi(h,a) = (Y,\alpha) \\
    1 \quad &\text{if } \psi(h,a) = (Y,\beta)
\end{cases}
\eeq

Let us have a fixed $B$ and $p_Y(sb) := \sum_{\th \in \H, \ta \in \A} p_Y(ha)B(ha|sb)$, $r(sb) := \sum_{\th \in \H, \ta \in \A} P(\pr|ha)B(ha|sb)\pr = \sum_{\th \in \H, \ta \in \A} r(ha)B(ha|sb)$. And, since $p_Y(ha)$ only depend on history, therefore, $p_Y(s) := p_Y(sb) = p_Y(sb')$ and the abstract reward function is defined as,

\beq\label{eq:example-reward-state}
r(sb) := \begin{cases}
    -\g p_Y(X) \quad &\text{if } (s,b) = (X,\alpha) \\
    -\frac{1}{2}-\g p_Y(X) \quad &\text{if } (s,b)= (X,\beta) \\
    \frac{1}{2}-\g p_Y(Y) \quad &\text{if } (s,b) = (Y,\alpha) \\
    1-\g p_Y(Y) \quad &\text{if } (s,b) = (Y,\beta).
\end{cases}
\eeq

The action-value function of the abstract-process takes the following form,
\beq
q^*(s,b) = r(sb) + \g p_Y(s)v^*(Y) + (1-p_Y(s))v^*(X).
\eeq

It is straight forward to see that $q^*(X, \alpha) > q^*(X, \beta)$ and $q^*(Y, \alpha) < q^*(Y, \beta)$. The history dependent conditional distribution of the abstract states, i.e. $p_Y(h)$, does not matter in this case. Therefore, the example homomorphism can handle non-MDP cases when $p_Y(h) \neq p_Y(\ph)$  for some $\psi(h)=\psi(\ph)$. For any $B$, the optimal policy in the abstract space is $\pi^*(X) = \alpha$ and $\pi^*(Y) = \beta$. From \eqref{eq:example-q}, for any $\psi(h)=s$, we get that $\Pi^*(h) \in \psi^{-1}_s(\pi^*(s))$.

\beq
V^*(h) = r(h\Pi^*(h)) + \g V^*(X)(1-P(Y|h\Pi^*(h))) + \g V^*(Y) P(Y|h\Pi^*(h))
\eeq

\beq
Q^*(h,a) = r(ha) + \g V^*(X)(1-P(Y|ha)) + \g V^*(Y) P(Y|ha)
\eeq

\beq
V^*(h) = r(h\Pi^*(h)) + \g P(Y|h\Pi^*(h))
\eeq

\beq
Q^*(h,a) = r(ha) + \g P(Y|ha)
\eeq

\section{Aggregation Example}\label{sec:aggregation-example}

\beq
\int_{R} B(\to\ta|sb) = 
\begin{cases}
2/3 & s=s_0, b=b_0, R=R_{4a} \\
1/3 & s=s_0, b=b_0, R=R_{4b} \\
1/2 & s=s_0, b=b_1, R=R_{3a} \\
1/2 & s=s_0, b=b_1, R=R_{3b} \\
1   & s=s_1, b=b_0, R=R_1 \\
1/2 & s=s_1, b=b_1, R=R_{2a} \\
1/2 & s=s_1, b=b_1, R=R_{2b} \\
0	& \text{otherwise.} 
\end{cases}
\eeq

\beq
P_\psi(s|oa \in R) =
\begin{cases}
1 & s=s_0, R=R_{2a} \\
1 & s=s_0, R=R_{2b} \\
1 & s=s_0, R=R_{3a} \\
1 & s=s_0, R=R_{3b} \\
1/2 & s=s_0, R=R_{4b} \\
1 & s=s_1, R=R_{1} \\
1 & s=s_1, R=R_{4a} \\
1/2 & s=s_1, R=R_{4b} \\
0 & \text{otherwise.}
\end{cases}
\eeq

\beq
p(\ps|sb) = \int P_\psi(\ps|\to\ta)B(\to\ta|sb)
\eeq

\beq
p(s_0|sb) = \int_{R_{2a}} B(\to\ta|sb) + \int_{R_{2b}} B(\to\ta|sb) + \int_{R_{3a}} B(\to\ta|sb) + \int_{R_{3b}} B(\to\ta|sb) + \frac{1}{2}\int_{R_{4b}} B(\to\ta|sb)
\eeq

\beq
p(s_1|sb) = \int_{R_{1}} B(\to\ta|sb) + \int_{R_{4a}} B(\to\ta|sb) + \frac{1}{2}\int_{R_{4b}} B(\to\ta|sb)
\eeq

\beq
p(s_0|sb) = \begin{bmatrix}
1/6 & 1 \\
0 & 1
\end{bmatrix}
\eeq

\beq
p(s_1|sb) = \begin{bmatrix}
5/6 & 0 \\
1 & 0
\end{bmatrix}
\eeq

\beq
p(s_00|sb) = B_{R_{2a}} + B_{R_{2b}} + B_{R_{3a}} + B_{R_{4b}}
\eeq

\beq
p(s_10|sb) = B_{R_{1}} + B_{R_{4b}}
\eeq

\beq
p(s_0\g|sb) = 0
\eeq

\beq
p(s_1\g|sb) = B_{R_{4a}}
\eeq

\beq
p(s_0\eps|sb) = B_{R_{3b}}
\eeq

\beq
p(s_1\eps|sb) = 0
\eeq

Since policy differs in region $R_{2a}$ and $R_{2b}$ by $\eps^\prime$. Therefore, lets take $\pi(b|s) = 1/2$ for all $s$ and $b$.
\beq
q^\pi(0,0) = 1/3(2\g + \eps) + \frac{\g}{2} ( 1/6 q^\pi(0,0) + 1/6 q^\pi(0,1) + 5/6 q(1,0) + 5/6 q(1,1) )
\eeq
\beq
q^\pi(0,1) = \eps/2 + \frac{\g}{2} \sum_{\ps,\tb} p(\ps | 01) q^\pi(\ps,\tb)
\eeq
\beq
q^\pi(1,0) = \frac{\g}{2} \sum_{\ps,\tb} p(\ps | 10) q^\pi(\ps,\tb)
\eeq
\beq
q^\pi(1,1) = \frac{\g}{2} \sum_{\ps,\tb} p(\ps | 11) q^\pi(\ps,\tb)
\eeq

\fi

\fi
\end{document}